\documentclass[journal]{IEEEtran}

\usepackage{stfloats}
\usepackage{amsfonts}
\usepackage{amssymb}
\usepackage{amsthm}
\usepackage{booktabs}
\usepackage{cite}
\usepackage[cmex10]{amsmath}
\usepackage{float}
\usepackage{color}
\usepackage{multirow}
\usepackage[normalem]{ulem}
\usepackage{tabularx}
\usepackage{subfigure}
\usepackage{fancybox,dashbox}
\usepackage{bbm}
\usepackage[dvipsnames]{xcolor}
\usepackage{balance}
\usepackage{algorithm}
\usepackage{algpseudocode}
\usepackage{enumitem}
\usepackage{bm}

\newtheorem{theorem}{Theorem}

\newtheorem{proposition}{Proposition}

\ifCLASSINFOpdf
\usepackage[pdftex]{graphicx}
\DeclareGraphicsExtensions{.pdf,.jpeg,.png}
\else
\usepackage[dvips]{graphicx}
\DeclareGraphicsExtensions{.eps}
\fi

\begin{document}

 
\title{TONIC: Token-Centric Semantic Communication for Task-Oriented Wireless Systems}
\author{Sige Liu,~\IEEEmembership{Member,~IEEE,} and Kezhi Wang,~\IEEEmembership{Senior Member,~IEEE}
\IEEEcompsocitemizethanks{
\IEEEcompsocthanksitem This work is supported in part by Eureka CELTIC-NEXT 5G4PHealth/Innovate UK project (10093679), UKRI under the Horizon Europe funding guarantee (EP/Y03743X/1), as part of the European Commission MSCA HarmonicAI project (101131117) and Royal Society project (IEC-NSFC-211264). K. Wang would like to acknowledge the support in part by the Royal Society Industry Fellow scheme (IF\textbackslash R2\textbackslash23200104). (Corresponding author: Kezhi Wang).
\IEEEcompsocthanksitem S. Liu and K. Wang are with the Department of Computer Science, Brunel University London, Uxbridge UB8 3PH, U.K. (e-mail:sige.liu@brunel.ac.uk; kezhi.wang@brunel.ac.uk).
}
}
\maketitle

\begin{abstract}
Tokens are becoming the basic units through which foundation models represent and process information for understanding and inference. However, traditional wireless communication, centered on bit-level fidelity, faces a mismatch between what is transmitted reliably and what downstream models actually consume. This mismatch calls for a communication design that directly accounts for token-level task relevance and downstream model requirements, rather than treating all transmitted bits as equally important.
In this paper, we propose TONIC, a token-centric semantic communication framework for task-oriented wireless systems. The transmitter converts each source sample into a sequence of tokens, estimates token-level task relevance, and allocates protection through utility-aware unequal error protection under a fixed channel-use budget. 
At the receiver, token-level confidence is used to gate unreliable decisions, turning harmful substitutions into recoverable erasures before a Transformer-based completion model restores the masked tokens for final task inference. 
Our framework combines transmitter-side semantic-aware protection with receiver-side confidence-aware gating in a modular and interpretable architecture, rather than relying solely on fully black-box end-to-end learning.
We further establish a utility-aware Bayes-risk interpretation for the receiver-side gating rule and study its interaction with unequal protection and completion. 
Experimental results on image classification show that TONIC consistently outperforms separation-based schemes, the pixel-domain DeepJSCC baseline, and token-domain baselines under matched communication budgets over AWGN, Rayleigh, and Rician channels.

\end{abstract}

\begin{IEEEkeywords}
Token communication, semantic communication, task-oriented communication, unequal error protection, generative completion, and foundation models.
\end{IEEEkeywords}

\section{Introduction}
Tokens are becoming the basic interface through which foundation models represent and process information \cite{vaswani2017attention,bai2025forgetbit}. In visual and multimodal systems, raw observations are increasingly mapped into token sequences or token grids that are directly consumed by downstream models for understanding, reasoning, and generation \cite{dosovitskiy2021vit,oord2017vqvae,esser2021taming}. For wireless systems, this means that the communicated object can no longer be viewed merely as a bitstream or a reconstructed signal. Instead, the central question is whether the semantic tokens required by the downstream model can be delivered reliably and efficiently \cite{wei2025unitocom,qiao2025todma}. This contrast between conventional bit-centric and the proposed token-centric communication is illustrated in Fig.~\ref{fig:framework_comparison}.

Traditional wireless communication remains centered on bit-level fidelity \cite{shannon1948math}. Yet once the receiver ultimately operates on tokens rather than reconstructed pixels or bitstreams, this design becomes increasingly mismatched to downstream processing \cite{gunduz2023beyond}. Different token positions may contribute very differently to the final task, so uniformly protecting all transmitted bits does not necessarily preserve the token positions that matter most to inference \cite{xin2024semanticSurvey}. Moreover, in token-based systems, an incorrect substitution can be substantially more harmful than an explicit erasure when a strong completion prior is available at the receiver \cite{chang2022maskgit,rombach2022ldm,guo2024diffusion}. These observations motivate a token-centric communication design that aligns the transmitted representation with the downstream model interface rather than optimizing bit fidelity alone.

This shift creates both opportunities and challenges. On the one hand, tokenized representations provide a structured semantic interface that is naturally compatible with modern generative and inference models \cite{esser2021taming,chang2022maskgit}. On the other hand, a practical token-centric wireless system is supposed to resolve several coupled design questions: how to quantify token-level task relevance, how to allocate unequal protection under a fixed channel-use budget, how to decide whether a decoded token should be trusted or erased, and how to exploit a completion prior without collapsing the communication pipeline into a fully black-box end-to-end training system. Resolving these questions is essential if token-centric communication is to be both effective and practically deployable.

Recent work has begun to move in this direction from several perspectives. In semantic and task-oriented communication, early studies showed that communication design should be aligned with meaning or downstream utility rather than exact symbol recovery \cite{xie2021deepsc,shao2022taskedge}. Related efforts further considered speech-oriented semantic communication \cite{weng2021speech}, multimodal task-oriented semantic communication \cite{xie2021vqa}, and explainable semantic communication \cite{ma2023taskExplainable}. In parallel, wireless image transmission has demonstrated the value of joint source channel design under bandwidth and channel uncertainty, starting from DeepJSCC \cite{bourtsoulatze2019deepjscc}, extending to bandwidth-agile \cite{kurka2021bandwidth} and constellation-constrained variants \cite{tung2022deepjsccq}, and more recently advancing through OFDM-adaptive \cite{wu2022ofdmjscc} and transformer-based architectures \cite{yang2025swinjscc}. At the same time, token-centric communication has emerged as a new direction in the era of large models, including information-bottleneck-based token communication \cite{wei2025unitocom}, token-domain multiple access \cite{qiao2025todma}, token-aware semantic-channel coding and modulation \cite{ying2026jsccmToken}, and selective or model-assisted token transmission \cite{peng2025selective,solat2025fedhlm}. 
Nevertheless, existing approaches still leave three important gaps: they do not explicitly integrate token-level task relevance into transmitter-side protection under a fixed communication budget; they do not clearly distinguish between accepting an unreliable token and erasing it at the receiver; and they do not provide a unified receiver rule that explicitly couples token utility, decoding confidence, and completion-assisted recovery into a coherent end-to-end framework.

\begin{figure*}[t]
    \centering
    \begin{minipage}[t]{\textwidth}
        \centering
        \includegraphics[width=0.8\textwidth]{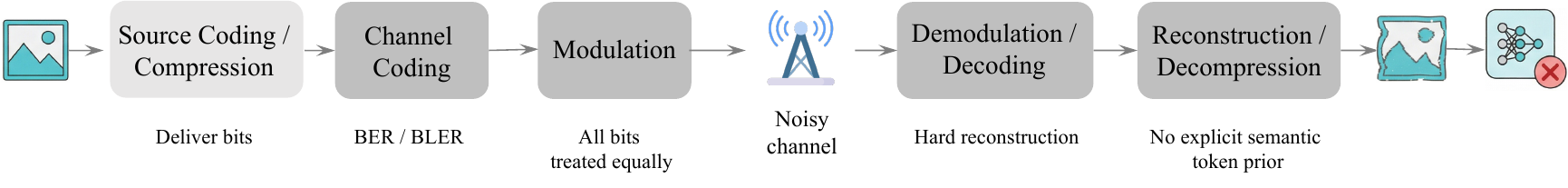}
        \vspace{0mm}
        
        {\footnotesize (a) Conventional bit-centric communication}
    \end{minipage}

    \vspace{2mm}

    \begin{minipage}[t]{\textwidth}
        \centering
        \includegraphics[width=0.8\textwidth]{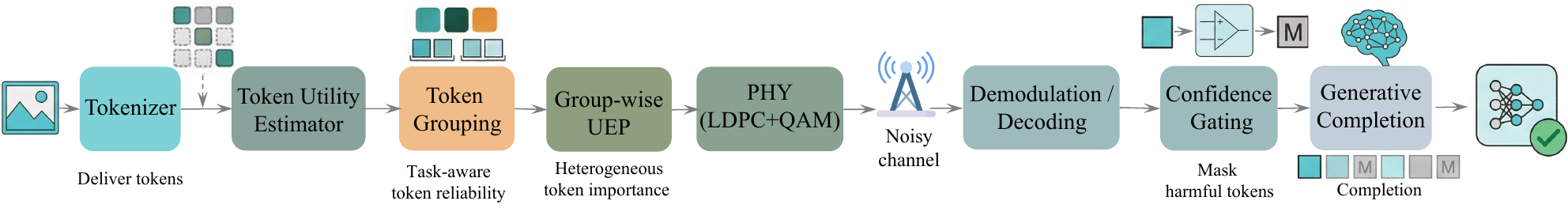}
        \vspace{0mm}
        
        {\footnotesize (b) Proposed token-centric TONIC framework}
    \end{minipage}
    \caption{Conventional bit-centric communication versus the proposed token-centric TONIC framework. }
    \label{fig:framework_comparison}
\end{figure*}

To address these gaps, we propose TONIC, a token-centric semantic communication framework for task-oriented wireless systems. In the image-transmission setting considered in this paper, the transmitter first converts each input image into a sequence of semantic tokens, estimates token-level task relevance, and allocates protection through utility-aware unequal error protection under a fixed channel-use budget. At the receiver, token-level confidence extracted from soft decoding is used to gate unreliable decisions, turning harmful substitutions into recoverable erasures before a Transformer-based completion model restores the masked tokens for final task inference. In this way, TONIC jointly determines which semantic tokens should receive stronger protection and how residual uncertainty should be handled at the receiver, while preserving a modular and interpretable architecture.
The main contributions of this paper are summarized as follows:
\begin{itemize}
    \item We develop a token-centric semantic communication framework for task-oriented wireless systems, in which the communicated object is a discrete sequence of visual tokens directly consumed by the downstream model rather than a reconstructed bitstream or pixel-domain signal.
    \item We propose a transmitter-side semantic-aware protection mechanism that combines token-utility estimation, utility grouping, and budget-constrained unequal error protection, enabling communication resources to be concentrated on task-critical token positions.
    \item We develop a receiver-side confidence-gating strategy that converts harmful low-confidence substitutions into recoverable erasures, and we establish a utility-aware Bayes-risk interpretation that links token utility, decoding confidence, and completion-assisted recovery.
    \item We integrate generative token completion into the communication loop and demonstrate through experiments that TONIC consistently outperforms separation-based transmission, pixel-domain DeepJSCC, and token-domain baselines under matched communication budgets across AWGN, Rayleigh, and Rician channels.
\end{itemize}

The remainder of this paper is organized as follows. Section~II reviews the most relevant work on semantic and task-oriented communication, wireless image transmission, token-centric communication, and generative recovery. Section~III introduces the system model and the high-level problem formulation. Section~IV presents the TONIC framework. Section~V develops the design and analysis of TONIC, including utility estimation and grouping, budget-constrained protection, confidence-aware gating, and offline threshold calibration. Section~VI presents the experimental setup and numerical results. Section~VII concludes the paper.

\section{Related Work}
This section reviews four research lines most relevant to the present work: semantic and task-oriented communication, learned wireless image transmission, token-centric communication, and generative recovery for semantic communication. Together, these lines provide the immediate context for understanding why token-aware protection, receiver-side acceptance or erasure decisions, and completion-assisted recovery should be studied in a unified framework.

\subsection{Semantic and Task-Oriented Communication}

Semantic and task-oriented communication move beyond conventional symbol-recovery objectives by aligning communication system design with meaning or downstream utility. Early semantic communication systems such as DeepSC demonstrated this principle for text transmission by optimizing sentence-level meaning recovery \cite{xie2021deepsc}. Similar ideas were later extended to speech-oriented semantic communication \cite{weng2021speech} and multimodal task-oriented communication for visual question answering \cite{xie2021vqa}. A more general task-oriented formulation for edge inference was developed in \cite{shao2022taskedge}, where communication was explicitly tied to the downstream inference objective. Explainability was further introduced into this line of work in \cite{ma2023taskExplainable}. 
These studies establish the importance of task-aware communication, but they do not directly resolve the token-centric setting considered here. In particular, they do not explicitly treat discrete semantic tokens as the communication object, nor do they address how token-level task relevance should be translated into unequal protection and receiver-side token acceptance or erasure decisions under a fixed symbol budget.

\subsection{Wireless Image Transmission}

In parallel, learned wireless image transmission has progressed rapidly through deep joint source channel coding. DeepJSCC first showed that end-to-end image transmission can outperform separation-based schemes in bandwidth-limited and noisy wireless settings \cite{bourtsoulatze2019deepjscc}. Bandwidth-agile JSCC later demonstrated adaptation to varying channel resources \cite{kurka2021bandwidth}. Digital or practical-constraint variants such as DeepJSCC-Q incorporated constellation constraints into learned JSCC \cite{tung2022deepjsccq}, while OFDM-adaptive designs introduced channel-adaptive transmission over multipath fading \cite{wu2022ofdmjscc}. More recently, transformer-based architectures such as SwinJSCC improved representation power and channel adaptation \cite{yang2025swinjscc}. Other recent digital or cooperative deep JSCC systems, such as D$^2$-JSCC and Process-and-Forward, further reflect the trend toward more practical and structured learned communication pipelines \cite{huang2025d2jscc,bian2025processforward}.
Despite their empirical performance, these approaches remain largely pixel- or feature-centric. The communicated object is typically an image, a continuous latent tensor, or a semantic feature representation, and the design objective is usually reconstruction fidelity, perceptual quality, or continuous feature preservation. By contrast, TONIC directly communicates semantic tokens and explicitly controls token-level protection under a fixed communication budget. Accordingly, the present work is not another image reconstruction architecture, but a token-centric communication framework instantiated and evaluated in an image classification setting.

\subsection{Token-Centric Communication}

The rise of foundation models has motivated a shift from bit-centric semantics toward token-centric communication. A semantic-information viewpoint centered on tokens was explicitly advocated in \cite{bai2025forgetbit}. UniToCom investigated token communication from an information-bottleneck perspective \cite{wei2025unitocom}, while ToDMA extended token-centric design to a multiple-access setting \cite{qiao2025todma}. Token-aware semantic-channel coding and modulation were studied in \cite{ying2026jsccmToken}, showing that token representations can be integrated into practical digital communication pipelines.
Other recent studies have explored related token-level mechanisms from different angles. Attention-guided semantic transmission was considered in \cite{lee2026attention}. Selective-token multimodal semantic communication was studied in \cite{peng2025selective}. Hybrid language-model-based token delivery was investigated in \cite{solat2025fedhlm}. 
These works strongly support the importance of token-aware communication, but they still leave open how transmitter-side protection, receiver-side token acceptance or erasure, and completion-assisted recovery should be jointly designed within a general framework.

\subsection{Generative Recovery and Completion-Assisted Communication}

Generative priors provide a natural mechanism for recovering incomplete semantic representations. In visual generative modeling, VQ-VAE introduced learned discrete latent tokens \cite{oord2017vqvae}, Taming Transformers demonstrated high-resolution token-based image generation \cite{esser2021taming}, and MaskGIT showed that missing visual tokens can be effectively restored from bidirectional context \cite{chang2022maskgit}. Latent diffusion models further reinforced the practical value of generative priors over structured latent spaces \cite{rombach2022ldm}. These developments suggest that erased token positions may be substantially easier to recover than wrong token substitutions when a strong completion prior is available.
This insight is increasingly relevant to semantic communication. Diffusion-assisted or generation-assisted semantic recovery has already been explored for semantically meaningful restoration under constrained wireless resources \cite{guo2024diffusion}. Language-oriented semantic communication with fine-tuned diffusion models was studied in \cite{wei2024language}. Generative semantic communication for joint image transmission and segmentation was developed in \cite{yuan2025generative}. More recently, foundation-model-based generative semantic communication has been used to analyze perception errors and semantic-aware power allocation \cite{xu2025genscfm}. 
However, these works do not explicitly provide a unified token-centric design that jointly addresses: 1) token-level task-aware protection at the transmitter, 2) a principled receiver rule for deciding whether a decoded token should be accepted or erased, and 3) completion-assisted recovery before downstream task inference.
TONIC differs from existing work in three aspects. First, it incorporates token-level task relevance directly into transmitter-side protection. Second, it introduces a receiver-side confidence-gating rule that explicitly distinguishes between accepting a decoded token and erasing it. Third, it combines this erasure-shaping mechanism with generative token completion and downstream task inference within a general modular framework.

\begin{figure*}[t]
    \centering
    \includegraphics[width=0.75\textwidth]{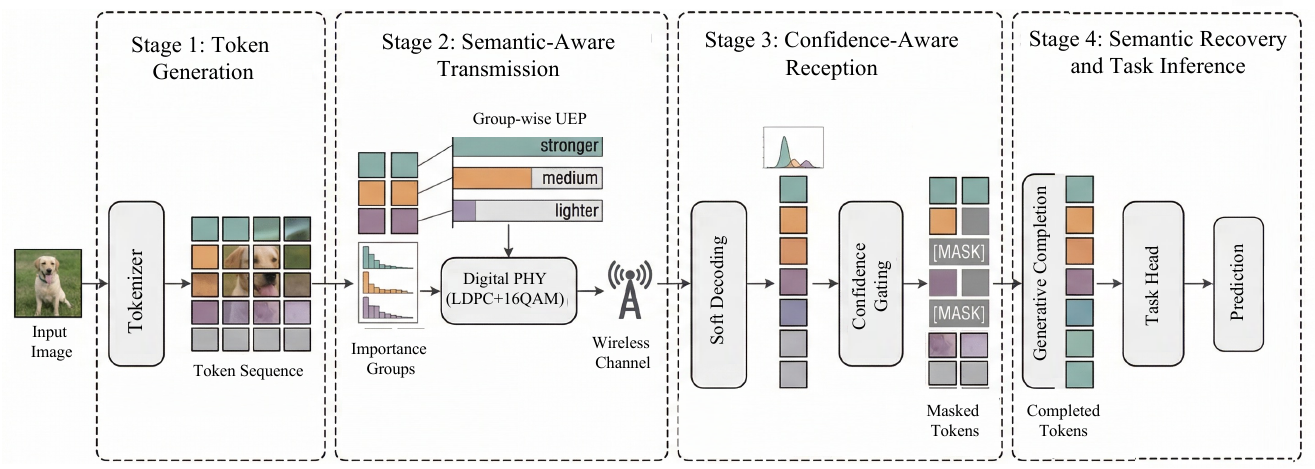}
    \caption{Online runtime workflow of TONIC. }
    \label{fig:tonic_runtime_workflow}
\end{figure*}

\section{System Model and Problem Formulation}
\label{sec:system_model}
As illustrated conceptually in Fig.~\ref{fig:framework_comparison} and at the system level in Fig.~\ref{fig:tonic_runtime_workflow}, we consider a token-centric task-oriented uplink communication system. A user device observes an image sample, converts it into a token sequence, and transmits it over a wireless channel under a fixed symbol budget. The receiver extracts token-level soft information from the received signal, applies confidence-aware token decisions with optional erasures, restores erased positions using a completion prior, and finally performs downstream task inference at the server.

Throughout the paper, bold lowercase letters denote vectors or stacked signal/token representations, bold uppercase letters denote matrices, and calligraphic letters denote sets. Scalars are written in standard italic form.

\subsection{Image-to-Token Representation}
\label{subsec:dataflow_tokens}

Let $\mathbf{x} \in \mathbb{R}^{H \times W \times C}$ denote the input image. A tokenizer $T(\cdot)$ maps $\mathbf{x}$ to a token sequence $\mathbf{t} = [t_1,\ldots,t_L]^{\mathsf T}$, where $t_i \in \mathcal{K} \triangleq \{1,2,\ldots,K\}$. Here, $L$ is the token sequence length and $K$ is the tokenizer codebook size. 

Let $\mathbf{E} \in \mathbb{R}^{K \times D}$ denote the token embedding table, where $D$ is the embedding dimension. The embedding of token position $i$ is $\mathbf{e}_i = \mathbf{E}[t_i,:] \in \mathbb{R}^{D}$, and the stacked embedding sequence is $\mathbf{Z} = [\mathbf{e}_1,\ldots,\mathbf{e}_L]^{\mathsf T} \in \mathbb{R}^{L \times D}$. In TONIC, the communicated object is the discrete token sequence $\mathbf{t}$, while the embedding sequence $\mathbf{Z}$ serves as the representation on which token-utility estimation, completion, and downstream task inference operate.

\subsection{Token-to-Waveform Mapping Under a Fixed Budget}
\label{subsec:token_waveform}

We allocate a fixed block of $N$ complex channel uses to each source sample. The transmitted baseband block is denoted by $\mathbf{s} = [s_1,\ldots,s_N]^{\mathsf T} \in \mathbb{C}^{N}$ and is subject to the average-power constraint
\begin{equation}
\frac{1}{N}\|\mathbf{s}\|_2^2 \le P.
\label{eq:power_constraint_v4}
\end{equation}
Here, $N$ denotes the per-sample communication budget, while $P$ denotes the average transmit-power constraint. Together, they define a unified transmission constraint for comparing separation-based transmission, pixel-domain DeepJSCC \cite{bourtsoulatze2019deepjscc}, and token-domain schemes on equal per-sample resources.
The end-to-end transmitter processing is abstracted as
\begin{equation}
\mathbf{s} = f_{\mathrm{tx}}(\mathbf{t};\boldsymbol{\pi}),
\label{eq:ftx_v4}
\end{equation}
where $f_{\mathrm{tx}}(\cdot)$ includes token-to-bit mapping, channel coding, digital modulation, and optional unequal protection, while $\boldsymbol{\pi}$ denotes the transmitter-side protection parameters to be specified later. At this stage, \eqref{eq:ftx_v4} only defines the communication interface between the token domain and the physical layer; the detailed design of the transmitter-side protection policy is developed in Section~V.

\subsection{Channel Model}
\label{subsec:channel_model}

We consider a flat block-fading complex baseband channel model,
\begin{equation}
\mathbf{r} = h\mathbf{s} + \mathbf{w},
\label{eq:channel_v4}
\end{equation}
where $\mathbf{s} \in \mathbb{C}^{N}$ and $\mathbf{r} \in \mathbb{C}^{N}$ denote the transmitted and received symbol blocks, respectively, $h \in \mathbb{C}$ is the channel coefficient assumed constant over the $N$-symbol block, and $\mathbf{w} \sim \mathcal{CN}(\mathbf{0},\sigma^2\mathbf{I})$ is circularly symmetric complex Gaussian noise. This model covers the channel instantiations considered later in the experiments in Section~VI, including AWGN as the special case $h \equiv 1$, Rayleigh block fading with $h \sim \mathcal{CN}(0,1)$, and normalized Rician block fading with a fixed $K$-factor. The receiver is assumed to have the channel-state information required for coherent demodulation and decoding. Unless otherwise stated, all channel models are normalized such that $\mathbb{E}[|h|^2]=1$, and the nominal average SNR is therefore $P/\sigma^2$.

\subsection{Receiver Soft Output and Token-Level Confidence}
\label{subsec:posterior_confidence}

From the received block $\mathbf{r}$, the receiver performs coherent demodulation and soft decoding to obtain token-level soft information. For token position $i$, let

\begin{equation}
p_i(k) \triangleq \Pr(\Theta_i = k \mid \mathbf{r},h), \qquad k \in \mathcal{K},
\label{eq:posterior_v4}
\end{equation}
denote the posterior distribution over the token alphabet, where $\Theta_i$ is the random source token at position $i$. In practice, $p_i(\cdot)$ is obtained by converting bit-level soft information, such as log-likelihood ratios, into posterior probabilities over the discrete token hypotheses under the fixed token-to-bit mapping.
Based on $p_i(\cdot)$, the receiver forms a hard token estimate $
\hat{t}_i = \arg\max_{k \in \mathcal{K}} p_i(k)
$ and the associated confidence score
$
c_i = \max_{k \in \mathcal{K}} p_i(k)
$.
Collecting these quantities over all positions yields the hard-decoded token sequence $\hat{\mathbf{t}} = [\hat{t}_1,\ldots,\hat{t}_L]^{\mathsf T}$ and the confidence sequence $\mathbf{c} = [c_1,\ldots,c_L]^{\mathsf T}$, which form the interface used later for receiver-side confidence-aware gating.

\begin{figure*}[t!]
    \centering
    \includegraphics[width=0.8\textwidth]{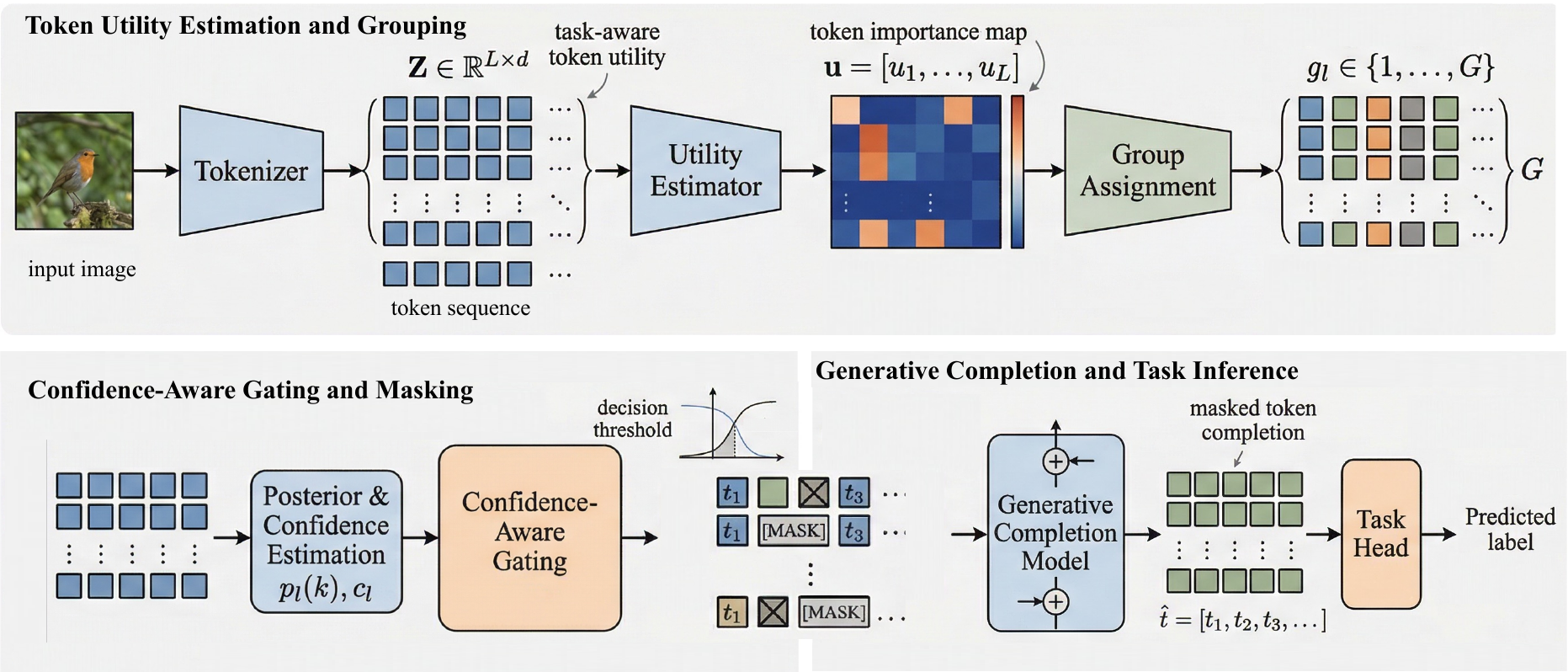}
    \caption{Mechanism decomposition of TONIC: utility-aware token grouping, confidence-aware gating and erasure shaping, and generative completion for task-oriented inference.}
    \label{fig:tonic_mechanisms}
\end{figure*}
\subsection{Erasure Interface, Completion, and Task Inference}
\label{subsec:erasure_completion}
To enable completion-assisted recovery, the receiver is allowed to output an abstract erasure symbol $\perp$ for unreliable positions. Let $\mathcal{K}_{\perp} \triangleq \mathcal{K} \cup \{\perp\}$ and define the post-gating token sequence as $\tilde{\mathbf{t}} = [\tilde{t}_1,\ldots,\tilde{t}_L]^{\mathsf T} \in \mathcal{K}_{\perp}^{L}$. Here, $\hat{\mathbf{t}}$ denotes the hard-decoded token sequence before gating, $\tilde{\mathbf{t}}$ denotes the gated sequence that may contain erasures, and $\bar{\mathbf{t}}$ denotes the final completed token sequence. At the system-model level, $\perp$ is an abstract erasure symbol external to the source token alphabet; in implementation, it is mapped to a dedicated mask token or mask embedding recognized by the completion model.
A completion prior then restores the erased positions according to
\begin{equation}
\bar{\mathbf{t}} = f_{\mathrm{comp}}(\tilde{\mathbf{t}}),
\label{eq:completion_v4}
\end{equation}
where $\bar{\mathbf{t}} = [\bar{t}_1,\ldots,\bar{t}_L]^{\mathsf T} \in \mathcal{K}^{L}$.

The server finally performs downstream inference using the completed embedding sequence $\bar{\mathbf{Z}} \in \mathbb{R}^{L \times D}$ induced by $\bar{\mathbf{t}}$ through the same embedding table $\mathbf{E}$, yielding
\begin{equation}
\hat{y} = f_{\mathrm{task}}(\bar{\mathbf{Z}}).
\label{eq:task_inference_v4}
\end{equation}
Thus, the end-to-end TONIC pipeline involves four token states: the source sequence $\mathbf{t}$, the hard-decoded sequence $\hat{\mathbf{t}}$, the gated sequence $\tilde{\mathbf{t}}$, and the final completed sequence $\bar{\mathbf{t}}$.

\subsection{Problem Formulation}
\label{subsec:problem_formulation}

Let $y$ denote the task-dependent ground truth and let $\hat{y}$ denote the final prediction at the server. We measure performance through a task loss $L_{\mathrm{task}}(\hat{y},y)$, such as cross-entropy for classification. Given a fixed tokenizer, a completion prior, and a downstream task head, our goal is to design the transmitter-side protection policy and the receiver-side decision rule so as to minimize the expected downstream task loss under a fixed communication budget.

Accordingly, the high-level design objective of TONIC is
\begin{equation}
(\boldsymbol{\pi}^{\star}, \Gamma^{\star})
=
\arg\min_{\boldsymbol{\pi},\,\Gamma}
\ \mathbb{E}_{(\mathbf{x},y),\,h,\,\mathbf{w}}
\!\left[
L_{\mathrm{task}}\!\left(
\hat{y}(\mathbf{x},h,\mathbf{w};\boldsymbol{\pi},\Gamma),\,y
\right)
\right],
\label{eq:problem_v4}
\end{equation}
where $\hat{y}(\mathbf{x},h,\mathbf{w};\boldsymbol{\pi},\Gamma)$ denotes the final task prediction induced by the end-to-end TONIC pipeline under the transmitter-side protection policy $\boldsymbol{\pi}$ and the receiver-side decision rule $\Gamma(\cdot)$. For notational simplicity, the dependence of $\hat{y}$ on the fixed tokenizer, completion model, and task head is suppressed. The objective in \eqref{eq:problem_v4} is defined under the fixed per-sample transmission budget $N$ and average-power constraint $P$ introduced in Section~III-B.

Problem \eqref{eq:problem_v4} is a system-level design objective rather than the training objective of a single end-to-end neural network. In particular, the tokenizer, completion model, and task head are trained offline, while the communication-specific design is realized through utility estimation, token grouping, budget-constrained protection, confidence-aware gating, and offline threshold calibration. The detailed framework and design are developed in Sections~IV and V. Directly solving \eqref{eq:problem_v4} as a unified optimization problem is challenging because the token representation and receiver decisions are discrete, the communication interfaces are non-differentiable, and transmitter-side protection, receiver-side uncertainty handling, and completion quality are tightly coupled under the fixed symbol budget. These difficulties motivate the modular TONIC design developed in the following sections.

\section{The TONIC Framework}
\label{sec:tonic_framework}
To address the system-level design objective, TONIC adopts a modular architecture with explicit interfaces between communication, confidence-aware token decisions, and downstream inference. The framework is organized around three coupled design components: transmitter-side protection, receiver-side confidence gating, and completion-assisted recovery before task inference. This section explains how these components are instantiated and how they interact within the overall TONIC pipeline, while the detailed utility definitions, protection design, and gating analysis are developed in Section~V.

\subsection{Framework Overview: Online Runtime and Offline Support}
\label{subsec:framework_overview}

TONIC is organized around two coupled workflows: an online runtime path for per-sample transmission and inference, and an offline preparation path that produces the artifacts required by the runtime system.

The online runtime workflow is illustrated in Fig.~\ref{fig:tonic_runtime_workflow}. Given an input image, the user equipment first tokenizes it into a token sequence. The sequence is then transmitted under a fixed communication budget using a group-wise protection profile. At the receiver, soft decoding produces token-level posterior information, which is converted into confidence-aware token decisions. Unreliable decisions are mapped to erasures, and the resulting masked token sequence is forwarded to a server-side completion model. The completed token sequence is finally passed to the downstream task head.

The offline preparation workflow is illustrated in Fig.~\ref{fig:tonic_offline_pipeline}. Before deployment, TONIC prepares several reusable artifacts, including a shared token-utility map, a utility-based grouping rule, calibrated group-wise protection profiles, and group-wise confidence thresholds. The completion model and task head are also trained offline and then frozen during communication experiments. This separation between offline preparation and online deployment keeps the runtime path lightweight while retaining semantic awareness and task alignment.

The key design principle is that communication uncertainty should be handled at two complementary levels. First, the transmitter should use the limited symbol budget to preferentially protect task-relevant token positions. Second, the receiver should avoid blindly trusting all hard token decisions; instead, it should convert sufficiently unreliable decisions into erasures whenever erasure is more compatible with completion-assisted recovery than direct acceptance

\begin{figure}[t!]
    \centering
    \includegraphics[width=\columnwidth]{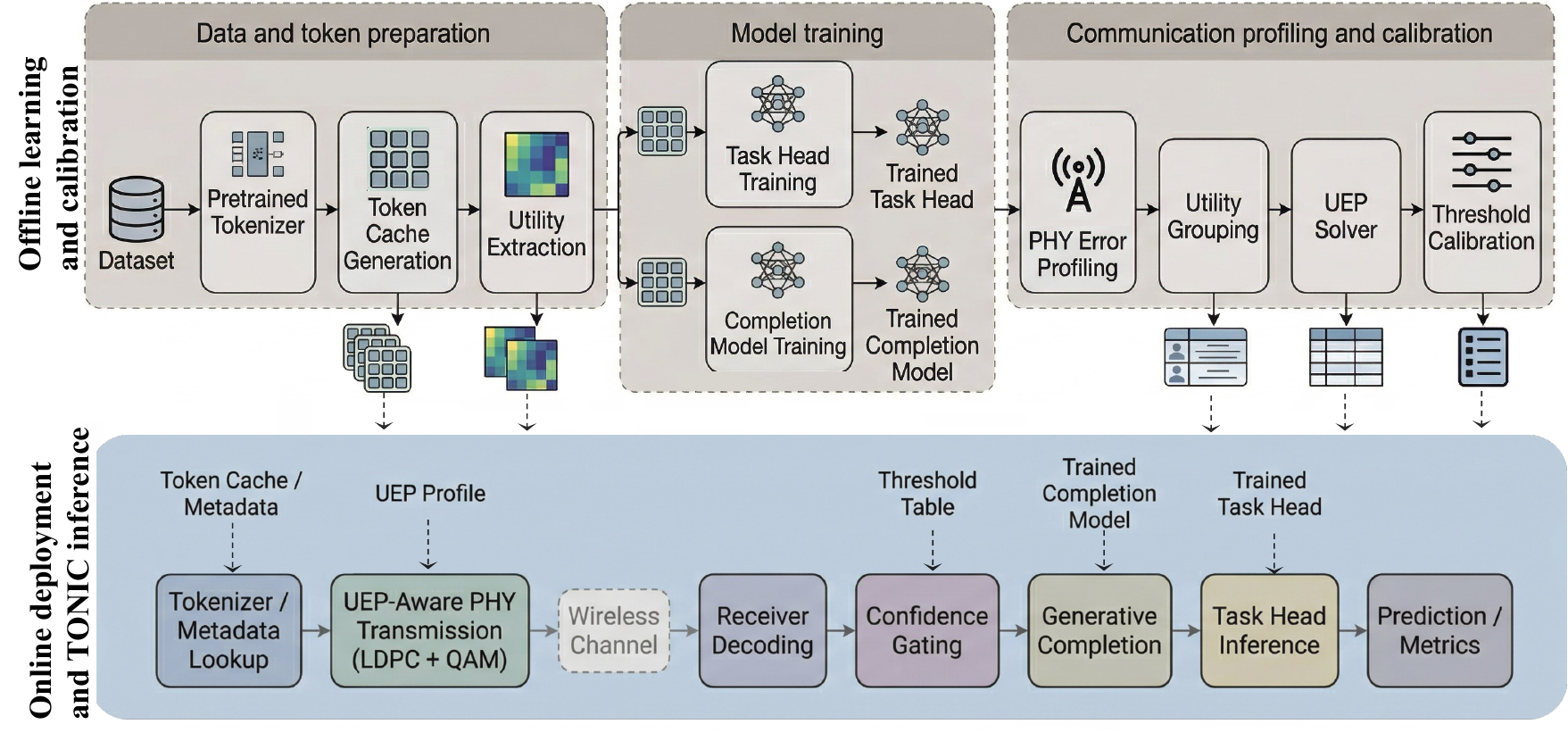}
    \caption{Offline learning and calibration pipeline.}
    \label{fig:tonic_offline_pipeline}
\end{figure}

\subsection{Transmitter-Side Semantic-Aware Protection}
\label{subsec:tx_framework}

The transmitter-side role of TONIC is to determine which token positions should receive stronger protection under the fixed communication budget. To this end, TONIC assigns each token position a utility score that reflects its relevance to the downstream task. Because per-position adaptation would be brittle and would increase control complexity, token positions are quantized into a small number of utility groups, as illustrated by the transmitter-side mechanism decomposition in Fig.~\ref{fig:tonic_mechanisms}. High-utility groups correspond to positions that are more critical to downstream inference, whereas low-utility groups contain positions that can tolerate weaker protection.

Once the token positions are grouped, TONIC selects a protection level for each group from a finite digital-PHY policy set. The key idea is not to optimize an unconstrained waveform encoder, but to choose a group-wise protection profile that remains compatible with standard digital communication modules. This policy set is realized through fixed 16QAM modulation together with multiple LDPC code-rate options. This preserves implementation compatibility while reflecting the central semantic insight that not all token positions should be treated equally.

At runtime, the transmitter applies the grouping rule and the calibrated group-wise protection profile. This avoids per-sample iterative optimization and keeps the runtime complexity at the user side modest. The detailed utility definitions, grouping procedure, and budget-constrained protection design are developed in Section~V.

\subsection{Receiver-Side Confidence Gating and Completion-Assisted Recovery}
\label{subsec:rx_framework}

The receiver-side role of TONIC is not merely to decode a hard token sequence, but to decide when a decoded token should be trusted and when it should instead be declared erased. This distinction is crucial in token-centric inference, because an incorrect token can be more harmful than an explicit erasure when a strong completion prior is available.

Starting from the token posteriors and hard decisions defined in Section~III, TONIC applies a confidence-aware gating rule to each token position. Tokens with sufficiently high confidence are accepted, while low-confidence decisions are converted into erasures. The thresholds are applied at the utility-group level, which yields a compact and robust receiver policy. In this way, the receiver explicitly transforms part of the residual uncertainty into a form that is more compatible with completion-assisted recovery.

The masked token sequence is then processed by a server-side generative completion model, which predicts plausible replacements for the erased tokens using bidirectional token context. The completed token sequence is mapped back to embeddings and passed to the task head for final inference. As illustrated by the receiver-side mechanism decomposition in Fig.~\ref{fig:tonic_mechanisms}, this design creates a direct interface between communication uncertainty and generative recovery. The transmitter reduces the probability of errors in important token positions, while the receiver converts part of the remaining uncertainty into a representation that can be more effectively repaired before task inference.

\subsection{Mechanism Decomposition and Role of Each Module}
\label{subsec:mechanism_roles}

Fig.~\ref{fig:tonic_mechanisms} summarizes the internal mechanism decomposition of TONIC. At a high level, the framework consists of three tightly coupled modules: utility-aware token grouping, budget-constrained unequal protection, and confidence-aware gating with completion-assisted recovery.

\begin{itemize}
    \item Utility-aware token grouping, which converts fine-grained token importance into a group structure that can be shared by both the transmitter and the receiver. Its role is to expose which token positions are more critical to downstream inference, while avoiding the fragility and control overhead of fully position-specific communication policies.
    \item Budget-constrained unequal protection, which allocates stronger protection to more important groups under the fixed symbol budget. Its role is to shape the error pattern before decoding by reducing the probability of harmful corruption on task-oriented token positions.
    \item Confidence-aware gating with completion, which operates after soft decoding. Its role is to prevent highly unreliable substitutions from being passed directly to the downstream model. Instead, sufficiently uncertain positions are converted into erasures and subsequently restored by the completion model before task inference.
\end{itemize}
These modules are complementary rather than redundant. Utility-aware grouping provides the semantic structure needed for resource allocation. Unequal protection reduces the frequency of harmful token errors before they reach the receiver. Confidence-aware gating and completion then handle part of the residual uncertainty in a way that is better aligned with downstream inference. The performance gain of TONIC, therefore, comes not from a single mechanism in isolation, but from the coordinated interaction of the three modules.

The framework above specifies how TONIC operates as a modular token-centric communication system, but it does not yet explain how token utility is quantified, how the group-wise protection profile is selected under a fixed communication budget, or how receiver-side confidence thresholds are set in a principled manner. These questions are addressed in Section~V, which develops the core design and analysis of TONIC in detail.


\section{Core Design and Analysis of TONIC}
\label{sec:tonic_design_analysis}
Section~IV describes TONIC at the framework level, including its online runtime workflow, offline preparation pipeline, and transmitter-receiver role decomposition. We now develop the core mechanisms that instantiate this framework, focusing on four questions: how to quantify token-level task relevance, how to map fine-grained token importance into a compact grouping interface, how to allocate protection under a fixed communication budget, and how to perform receiver-side confidence-aware gating in a principled manner. We also present the offline threshold calibration procedure used to obtain a deployable receiver policy.

\subsection{Token Utility Estimation and Grouping}
\label{subsec:tonic_utility_grouping}
A central design principle of TONIC is that token positions do not contribute equally to the downstream task. We therefore associate each token position $i$ with a utility score, where a larger value indicates that corruption at that position is expected to induce a larger degradation in downstream task performance. Since TONIC is designed for practical deployment, it is important to distinguish between sample-wise utility measures, which define how token importance is assessed for an individual sample, and the shared position-wise utility profile actually used for grouping and protection design during deployment.

\subsubsection{Sample-wise gradient-based utility}
\label{subsubsec:tonic_grad_utility}

For a given sample $(\mathbf{x},y)$, let $\mathbf{t}=T(\mathbf{x})$ denote the source token sequence and let $\mathbf{e}_i=\mathbf{E}[t_i,:]$ denote the embedding vector at token position $i$. A practical utility proxy is the sensitivity of the task loss to the corresponding token embedding:
\begin{equation}
w_i^{\mathrm{grad}}(\mathbf{x},y)
\triangleq
\left\|
\frac{\partial L_{\mathrm{task}}}{\partial \mathbf{e}_i}
\right\|_2 .
\label{eq:tonic_grad_utility}
\end{equation}
This quantity measures how strongly the task loss changes under a local perturbation of the embedding at position $i$. It can be computed efficiently through standard back-propagation and therefore serves as the practical utility signal used by TONIC.

\subsubsection{Sample-wise masking-based utility}
\label{subsubsec:tonic_mask_utility}

We also consider a stronger but more expensive intervention-based utility for offline reference. Let $\mathbf{Z}\in\mathbb{R}^{L\times D}$ denote the clean embedding sequence induced by $\mathbf{t}$, and let $\mathbf{Z}^{(i\leftarrow \perp)}$ denote the sequence obtained by replacing position $i$ with a learned mask embedding. The masking-based utility is defined as
\begin{equation}
w_i^{\mathrm{mask}}(\mathbf{x},y)
\triangleq
L_{\mathrm{task}}\!\left(f_{\mathrm{task}}\!\left(\mathbf{Z}^{(i\leftarrow \perp)}\right),y\right)
-
L_{\mathrm{task}}\!\left(f_{\mathrm{task}}(\mathbf{Z}),y\right).
\label{eq:tonic_mask_utility}
\end{equation}
This score directly quantifies the increase in task loss caused by removing the information at token position $i$. In TONIC, the masking-based utility is used only as an oracle-aided offline reference and ablation target, rather than as the default deployable utility measure.

\subsubsection{Shared utility profile for deployment}
\label{subsubsec:tonic_shared_utility}

TONIC does not recompute utility scores or regroup token positions on a per-sample basis during deployment. Instead, it uses a shared position-wise utility profile estimated offline from a calibration set $\mathcal{D}_{\mathrm{val}}$. Specifically, the deployment-time gradient-based utility profile is defined as

\begin{equation}
\bar{w}_i^{\mathrm{grad}}
\triangleq
\frac{1}{|\mathcal{D}_{\mathrm{val}}|}
\sum_{(\mathbf{x},y)\in\mathcal{D}_{\mathrm{val}}}
w_i^{\mathrm{grad}}(\mathbf{x},y),
\label{eq:tonic_grad_utility_avg}
\end{equation}
and, for the oracle-aided reference,
\begin{equation}
\bar{w}_i^{\mathrm{mask}}
\triangleq
\frac{1}{|\mathcal{D}_{\mathrm{val}}|}
\sum_{(\mathbf{x},y)\in\mathcal{D}_{\mathrm{val}}}
w_i^{\mathrm{mask}}(\mathbf{x},y).
\label{eq:tonic_mask_utility_avg}
\end{equation}
These shared position-wise profiles are the quantities actually used for grouping and protection design in deployment. In the sequel, the notation $\bar{w}_i$ refers generically to the deployment-time utility profile, instantiated either by $\bar{w}_i^{\mathrm{grad}}$ for the practical design or by $\bar{w}_i^{\mathrm{mask}}$ for the oracle-aided reference.

\subsubsection{Utility quantization and grouping}
\label{subsubsec:tonic_grouping}

Assigning an independent protection parameter to every token position would significantly increase control complexity. TONIC therefore quantizes token positions into a small number of utility groups. Specifically, let $g(i)\in\{1,\ldots,G\}$ denote the group index of token position $i$, where $G$ is the number of utility groups.
Let $L_g$ denote the number of token positions in group $g$, and let $W_g$ denote the aggregate utility mass of that group:
\begin{equation}
L_g \triangleq \left|\left\{i:g(i)=g\right\}\right|,
\qquad
W_g \triangleq \sum_{i:g(i)=g}\bar{w}_i .
\label{eq:tonic_group_stats}
\end{equation}
This grouping step reduces fine-grained token-importance heterogeneity to a compact shared interface that can be consistently used by both the transmitter and the receiver. At the transmitter, the grouping map determines how communication resources are allocated across token subsets. At the receiver, the same grouping structure supports robust group-wise confidence thresholds.

\subsection{Budget-Constrained Utility-Aware Protection}
\label{subsec:tonic_uep}

Once the utility groups are fixed, the transmitter only needs to assign one protection level to each group. TONIC performs this design over a finite digital-PHY policy set $\mathcal{P}=\{\pi^{(1)},\ldots,\pi^{(|\mathcal{P}|)}\}$. In the present implementation, each policy corresponds to a practical operating point under fixed 16QAM modulation and a discrete choice of LDPC code rate. Let $c(\pi)$ denote the symbol cost per token under policy $\pi$, and let $\widehat{\varepsilon}_g(\pi)$ denote the offline-profiled post-decoding token error rate of group $g$ under policy $\pi$, measured before receiver-side gating and completion. 
In practice, these error curves are profiled offline under the target deployment condition, including the channel model and operating point used for protection design.
The dependence on $g$ reflects the fact that different token groups may exhibit different effective reliability statistics under the same PHY operating point.
The full design objective in \eqref{eq:problem_v4} jointly couples protection, receiver-side gating, completion, and downstream inference, and is not directly tractable. TONIC therefore adopts a practical utility-weighted surrogate for transmitter-side protection design:
\begin{equation}
\begin{aligned}
\min_{\{\pi_g\}_{g=1}^{G}}
\quad &
\sum_{g=1}^{G} W_g \, \widehat{\varepsilon}_g(\pi_g) \\
\mathrm{s.t.}\quad &
\sum_{g=1}^{G} L_g \, c(\pi_g) \le N, \qquad
\pi_g \in \mathcal{P},\ \forall g ,
\end{aligned}
\label{eq:tonic_surrogate_knapsack}
\end{equation}
where $W_g$ is the aggregate utility mass of group $g$ and $L_g$ is the number of token positions in that group. This design criterion prioritizes reliability improvements in groups with larger downstream importance while respecting the fixed symbol budget $N$.

\begin{algorithm}[t]
\caption{Utility-Weighted Group-Wise UEP Scheduler}
\label{alg:tonic_uep_scheduler}
\begin{algorithmic}[1]
\Require Group utility masses $\{W_g\}_{g=1}^{G}$, group sizes $\{L_g\}_{g=1}^{G}$, total budget $N$, policy set $\mathcal{P}$ with costs $c(\pi)$, and profiled error curves $\widehat{\varepsilon}_g(\pi)$
\Ensure Group-wise protection profile $\{\pi_g\}_{g=1}^{G}$
\State Initialize each group with the least costly feasible policy
\State Compute the remaining budget after initialization
\While{there exists a feasible upgrade within the remaining budget}
    \State For each group and each feasible policy upgrade, compute the utility-weighted reduction in surrogate loss per additional symbol
    \State Select the upgrade with the largest positive gain-to-cost ratio
    \If{no positive-gain upgrade exists}
        \State \textbf{break}
    \EndIf
    \State Apply the selected upgrade and update the remaining budget
\EndWhile
\State \Return the final group-wise protection profile $\{\pi_g\}_{g=1}^{G}$
\end{algorithmic}
\end{algorithm}

The optimization in \eqref{eq:tonic_surrogate_knapsack} should be interpreted as a practical transmitter-side design problem rather than an exact reformulation of the full end-to-end objective. Its purpose is to translate the shared utility profile into a group-wise protection profile that is compatible with standard digital communication modules.

Algorithm~\ref{alg:tonic_uep_scheduler} provides a practical realization of \eqref{eq:tonic_surrogate_knapsack}. Starting from the least costly feasible protection profile, it incrementally allocates additional symbols to the upgrades that yield the largest utility-weighted reliability gain per unit cost. The resulting group-wise protection profile is then fixed during deployment.

The utility-weighted surrogate in \eqref{eq:tonic_surrogate_knapsack} is motivated by the following sample-level upper bound, which links task-loss degradation to a utility-weighted token error count before receiver-side gating and completion.

\begin{proposition}[Utility-weighted upper bound on task-loss degradation]
\label{prop:tonic_utility_bound}
Let $\mathbf{Z}=[\mathbf{e}_1,\ldots,\mathbf{e}_L]^{\mathsf T}$ denote the clean embedding sequence induced by the source token sequence $\mathbf{t}$, and let $\hat{\mathbf{Z}}=[\hat{\mathbf{e}}_1,\ldots,\hat{\mathbf{e}}_L]^{\mathsf T}$ denote the embedding sequence induced by the hard-decoded token sequence $\hat{\mathbf{t}}$, where $\hat{\mathbf{e}}_i \triangleq \mathbf{E}[\hat{t}_i,:]$. Define the interpolation path
\begin{equation}
\mathbf{Z}(\alpha)\triangleq \mathbf{Z}+\alpha(\hat{\mathbf{Z}}-\mathbf{Z}),
\qquad \alpha\in[0,1].
\label{eq:tonic_interp_path}
\end{equation}
Assume that the task loss is differentiable with respect to the embedding sequence and that the embedding-table diameter is bounded by
\begin{equation}
\|\mathbf{E}[a,:]-\mathbf{E}[b,:]\|_2 \le \Delta_{\max},
\qquad \forall a,b \in \mathcal{K}.
\label{eq:tonic_embedding_diameter}
\end{equation}
Further define the path-dependent sensitivity
\begin{equation}
w_i^{\mathrm{sup}}
\triangleq
\sup_{\alpha\in[0,1]}
\left\|
\frac{\partial L_{\mathrm{task}}(f_{\mathrm{task}}(\mathbf{Z}(\alpha)),y)}
{\partial \mathbf{e}_i}
\right\|_2 .
\label{eq:tonic_wsup}
\end{equation}
Then the task-loss degradation caused by hard token decoding satisfies
\begin{equation}
\begin{aligned}
&\left|
L_{\mathrm{task}}(f_{\mathrm{task}}(\hat{\mathbf{Z}}),y)
-
L_{\mathrm{task}}(f_{\mathrm{task}}(\mathbf{Z}),y)
\right|  \\
&~~~~~~~~~~~~~\le
\Delta_{\max}
\sum_{i=1}^{L}
w_i^{\mathrm{sup}}\,\mathbbm{1}\{\hat{t}_i\neq t_i\}.
\label{eq:tonic_loss_upper_bound}
\end{aligned}
\end{equation}

\end{proposition}

\begin{proof}
Define
\[
\phi(\alpha)\triangleq L_{\mathrm{task}}(f_{\mathrm{task}}(\mathbf{Z}(\alpha)),y).
\]
By the fundamental theorem of calculus,
\[
\phi(1)-\phi(0)=\int_{0}^{1}\phi'(\alpha)\,d\alpha .
\]
Using the chain rule,
\[
\phi'(\alpha)
=
\sum_{i=1}^{L}
\left\langle
\frac{\partial L_{\mathrm{task}}(f_{\mathrm{task}}(\mathbf{Z}(\alpha)),y)}{\partial \mathbf{e}_i},
\hat{\mathbf{e}}_i-\mathbf{e}_i
\right\rangle .
\]
Taking absolute values, applying the triangle inequality and Cauchy--Schwarz, and using the definition of $w_i^{\mathrm{sup}}$ yields
\[
|\phi(1)-\phi(0)|
\le
\sum_{i=1}^{L}
w_i^{\mathrm{sup}}\,\|\hat{\mathbf{e}}_i-\mathbf{e}_i\|_2 .
\]
If $\hat{t}_i=t_i$, then $\hat{\mathbf{e}}_i=\mathbf{e}_i$. Otherwise, \eqref{eq:tonic_embedding_diameter} implies
\[
\|\hat{\mathbf{e}}_i-\mathbf{e}_i\|_2 \le \Delta_{\max}.
\]
Substituting this bound proves \eqref{eq:tonic_loss_upper_bound}.
\end{proof}

Proposition~\ref{prop:tonic_utility_bound} justifies utility-weighted error counting as a transmitter-side design criterion, that is, token errors on task-sensitive positions contribute more strongly to the upper bound on task-loss degradation. The proposition is intentionally stated at the pre-gating, pre-completion stage, since its role is to motivate the transmitter-side protection surrogate rather than to characterize the full end-to-end TONIC pipeline. In practice, the path-dependent quantity $w_i^{\mathrm{sup}}$ is not directly tractable, so TONIC uses the gradient-based utility in \eqref{eq:tonic_grad_utility} as a practical first-order proxy and then aggregates it into the shared deployment-time profile defined in \eqref{eq:tonic_grad_utility_avg}.

\subsection{Confidence-Aware Gating and Completion-Assisted Recovery}
\label{subsec:tonic_gating}

The receiver-side goal in TONIC is not merely to output a hard token estimate, but to decide whether the decoded token should be trusted or erased. This distinction is crucial in token-centric inference because an incorrect token can be more harmful than an explicit erasure when a strong completion prior is available. 
In the following development, $w_i$ denotes the effective token utility associated with position $i$, as induced by the deployment-time shared utility profile introduced earlier.

For token position $i$, let the receiver choose an action $a_i \in \mathcal{K}_{\perp}$, where $\mathcal{K}_{\perp}=\mathcal{K}\cup\{\perp\}$. We assign zero cost to a correct accepted token, cost $w_i$ to an incorrect accepted token, and cost $\lambda_i$ to an erasure, where $\lambda_i$ represents the effective penalty of deferring the decision to completion-assisted recovery. Under the posterior distribution $p_i(k)$ defined in \eqref{eq:posterior_v4}, the conditional Bayes risk of outputting token $u\in\mathcal{K}$ is
\begin{equation}
R_i(u)=w_i\bigl(1-p_i(u)\bigr),
\label{eq:tonic_token_risk}
\end{equation}
while the Bayes risk of erasing the position is
\begin{equation}
R_i(\perp)=\lambda_i .
\label{eq:tonic_erasure_risk}
\end{equation}

\begin{theorem}[Utility-aware Bayes-optimal confidence threshold]
\label{thm:tonic_gating}
Assume $w_i>0$ and $0\le\lambda_i\le w_i$. Let $\hat{t}_i=\arg\max_{k\in\mathcal{K}}p_i(k)$ and $c_i=\max_{k\in\mathcal{K}}p_i(k)$. Then the Bayes-optimal action is either the MAP token or an erasure:
\begin{equation}
a_i^\star=
\begin{cases}
\hat{t}_i, & \text{if } w_i(1-c_i)\le\lambda_i,\\
\perp, & \text{otherwise}.
\end{cases}
\label{eq:tonic_opt_action}
\end{equation}
Equivalently, the rule can be written as confidence thresholding:
\begin{equation}
c_i\ge\tau_i \Rightarrow \hat{t}_i, \qquad
c_i<\tau_i \Rightarrow \perp,
\qquad
\tau_i \triangleq 1-\frac{\lambda_i}{w_i}\in[0,1].
\label{eq:tonic_opt_tau}
\end{equation}
\end{theorem}

\begin{proof}
For any token output $u\in\mathcal{K}$, the risk in \eqref{eq:tonic_token_risk} is minimized by choosing the MAP token $\hat{t}_i$, which yields the minimum token-output risk $w_i(1-c_i)$. The erasure action has risk $\lambda_i$. Therefore, the Bayes-optimal decision is to accept the MAP token if and only if $w_i(1-c_i)\le\lambda_i$, and to erase otherwise, which proves \eqref{eq:tonic_opt_action}. Since $w_i>0$, this inequality is equivalent to thresholding $c_i$ at $\tau_i=1-\lambda_i/w_i$. The assumption $0\le\lambda_i\le w_i$ guarantees that $\tau_i\in[0,1]$.
\end{proof}

Theorem~\ref{thm:tonic_gating} establishes the form of the receiver-side gating rule: a decoded token should be accepted only when its utility-weighted substitution risk is no larger than the effective erasure penalty after completion. The theorem is not intended to provide a closed-form calibration recipe for deployment. Instead, it explains why confidence-aware gating should take a threshold form and why the threshold should depend on both token utility and the relative value of erasure versus direct acceptance.


Using a distinct threshold for every token position would be unnecessary and would increase control complexity. TONIC therefore deploys group-wise thresholds, so that all positions in the same utility group share one confidence threshold. The receiver rule becomes
\begin{equation}
\tilde{t}_i =
\begin{cases}
\hat{t}_i, & c_i \ge \tau_{g(i)},\\
\perp, & c_i < \tau_{g(i)}.
\end{cases}
\label{eq:tonic_gating_rule}
\end{equation}
This grouped parameterization is more robust, reduces control overhead, and naturally aligns the receiver-side policy with the transmitter-side utility grouping.
Algorithm~\ref{alg:online_receiver} summarizes the complete online receiver-side procedure of TONIC, including confidence gating, completion, and downstream task inference.

\begin{algorithm}[t]
\caption{Online Receiver Procedure: Confidence Gating, Completion, and Task Inference}
\label{alg:online_receiver}
\begin{algorithmic}[1]
\Require Received signal $\mathbf{r}$, channel coefficient $h$, grouping map $g(i)$, group-wise thresholds $\{\tau_g\}_{g=1}^{G}$, completion model $f_{\mathrm{comp}}$, task head $f_{\mathrm{task}}$
\Ensure Final prediction $\hat{y}$, hard token sequence $\hat{\mathbf{t}}$, gated token sequence $\tilde{\mathbf{t}}$, completed token sequence $\bar{\mathbf{t}}$
\State Perform coherent demodulation and soft decoding using $(\mathbf{r}, h)$
\State Obtain token posterior distributions $\{p_i(k)\}_{i=1}^{L}$ over $k\in\mathcal{K}$
\For{$i=1$ to $L$}
    \State Compute hard token estimate $\hat{t}_i \gets \arg\max_{k\in\mathcal{K}} p_i(k)$
    \State Compute confidence score $c_i \gets \max_{k\in\mathcal{K}} p_i(k)$
    \If{$c_i \ge \tau_{g(i)}$}
        \State $\tilde{t}_i \gets \hat{t}_i$
    \Else
        \State $\tilde{t}_i \gets \perp$
    \EndIf
\EndFor
\State Form the gated token sequence $\tilde{\mathbf{t}}=[\tilde{t}_1,\ldots,\tilde{t}_L]^{\mathsf T}$
\State Restore erased positions by completion: $\bar{\mathbf{t}} \gets f_{\mathrm{comp}}(\tilde{\mathbf{t}})$
\State Convert $\bar{\mathbf{t}}$ to the completed embedding sequence $\bar{\mathbf{Z}}$
\State Perform downstream inference: $\hat{y} \gets f_{\mathrm{task}}(\bar{\mathbf{Z}})$
\State \Return $\hat{y},\hat{\mathbf{t}},\tilde{\mathbf{t}},\bar{\mathbf{t}}$
\end{algorithmic}
\end{algorithm}

\subsection{Offline Calibration of Group-Wise Thresholds}
\label{subsec:tonic_calibration}

Theorem~\ref{thm:tonic_gating} characterizes the form of the receiver-side gating rule, but the effective erasure penalty $\lambda_i$ is not directly available in closed form. This is because the value of erasing a token position depends jointly on the completion prior, the downstream task head, and the operating point of the overall communication pipeline. Accordingly, TONIC does not attempt to compute token-wise thresholds from \eqref{eq:tonic_opt_tau} directly. Instead, it calibrates group-wise confidence thresholds offline by minimizing the validation task loss.

In practice, this calibration is carried out for the target deployment condition associated with the chosen protection profile and channel operating point.
Given a fixed tokenizer, grouping rule, protection profile, completion model, and task head, the thresholds $\{\tau_g\}_{g=1}^{G}$ are selected by solving
\begin{equation}
\min_{\{\tau_g \in [0,1]\}}
\frac{1}{M}
\sum_{m=1}^{M}
L_{\mathrm{task}}
\bigl(
\hat{y}^{(m)}(\{\tau_g\}), y^{(m)}
\bigr),
\label{eq:tonic_tau_calib_obj}
\end{equation}
where $\hat{y}^{(m)}(\{\tau_g\})$ denotes the final task prediction for validation sample $m$ after confidence gating, completion, and task inference. In other words, the calibration step directly optimizes the deployment-time receiver policy against the final downstream objective, while holding all other system components fixed.
In practice, TONIC solves \eqref{eq:tonic_tau_calib_obj} approximately by coordinate search over a finite threshold grid. This avoids introducing a fragile inner optimization loop at runtime while providing a stable and reproducible receiver policy for deployment. The resulting procedure is summarized in Algorithm~\ref{alg:tonic_tau_calibration}.

\begin{algorithm}[t]
\caption{Offline Calibration of Group-Wise Thresholds}
\label{alg:tonic_tau_calibration}
\begin{algorithmic}[1]
\Require Validation set $\{(\mathbf{x}^{(m)},y^{(m)})\}_{m=1}^{M}$, grouping rule $g(i)$, fixed protection profile, completion model $f_{\mathrm{comp}}$, task head $f_{\mathrm{task}}$, threshold candidate grid $\mathcal{T}_{\mathrm{grid}}$, number of coordinate-search passes $J$
\Ensure Calibrated thresholds $\{\tau_g\}_{g=1}^{G}$
\State Initialize $\tau_g \leftarrow 0.5$ for all $g$
\For{$j=1$ to $J$}
    \For{$g=1$ to $G$}
        \State Temporarily fix all thresholds except $\tau_g$
        \State Search over $\tau\in\mathcal{T}_{\mathrm{grid}}$ and evaluate the average validation task loss
        \State Update $\tau_g$ with the value yielding the smallest validation loss
    \EndFor
\EndFor
\State \Return $\{\tau_g\}_{g=1}^{G}$
\end{algorithmic}
\end{algorithm}

This calibration procedure is the deployment-oriented counterpart of the Bayes-risk interpretation in Theorem~\ref{thm:tonic_gating}. Increasing $\tau_g$ declares more low-confidence positions erased and can reduce harmful substitutions, while decreasing $\tau_g$ passes more hard token decisions directly to the downstream model. The calibrated thresholds are then used by the online receiver procedure in Algorithm~\ref{alg:online_receiver}.

\subsection{Coupling, Deployment Complexity, and Design Implications}
\label{subsec:tonic_coupling_complexity}

The key design insight of TONIC is that unequal protection and confidence-aware gating are complementary rather than competing mechanisms. Unequal protection acts before decoding by reducing the probability of harmful corruption on task-relevant token positions. Confidence-aware gating acts after decoding by preventing highly unreliable substitutions from being passed directly to the downstream model. Completion-assisted recovery then handles part of the remaining uncertainty by restoring erased positions using contextual token priors. Accordingly, the gain of TONIC comes from the coordinated interaction of these mechanisms rather than from any single module in isolation.

This interaction suggests a practical design principle. Token groups with higher downstream importance should generally receive both stronger transmission protection and more conservative receiver-side acceptance, so that residual uncertainty on important positions is preferentially converted into recoverable erasures rather than accepted as low-confidence substitutions. By contrast, less critical groups can tolerate weaker protection and more permissive acceptance without causing the same level of task degradation. This principle explains why the transmitter-side utility profile and the receiver-side confidence thresholds should be designed jointly.

From a deployment perspective, the main additional complexity of TONIC lies in offline preparation rather than online runtime. Utility estimation, grouping, protection profiling, and threshold calibration are all performed offline. During runtime, the user equipment only applies a precomputed grouping rule and a fixed group-wise protection profile, while the receiver performs lightweight confidence thresholding before forwarding the masked token sequence to the server-side completion model. This separation preserves practical deployability while retaining the semantic advantages of token-aware protection and completion-assisted recovery.

\begin{figure*}[t!]
    \centering
    \subfigure[AWGN.]{
        \includegraphics[width=0.31\textwidth]{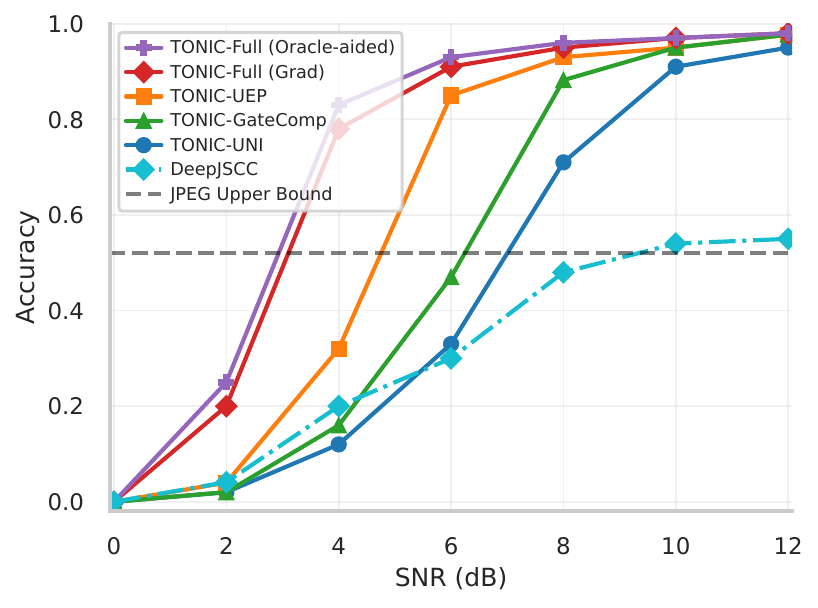}
        \label{fig:main_acc_awgn}
    }
    \hfill
    \subfigure[Rayleigh fading.]{
        \includegraphics[width=0.31\textwidth]{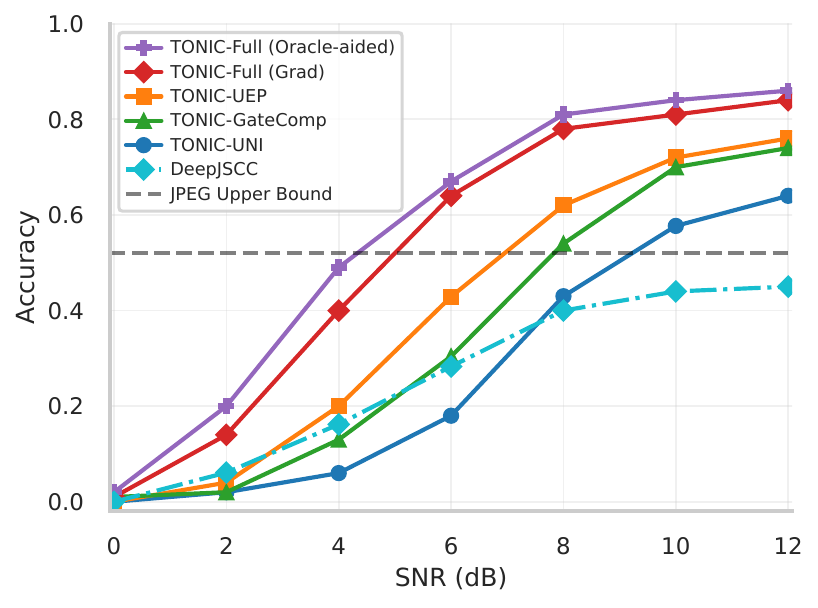}
        \label{fig:main_acc_rayleigh}
    }
    \hfill
    \subfigure[Rician fading.]{
        \includegraphics[width=0.31\textwidth]{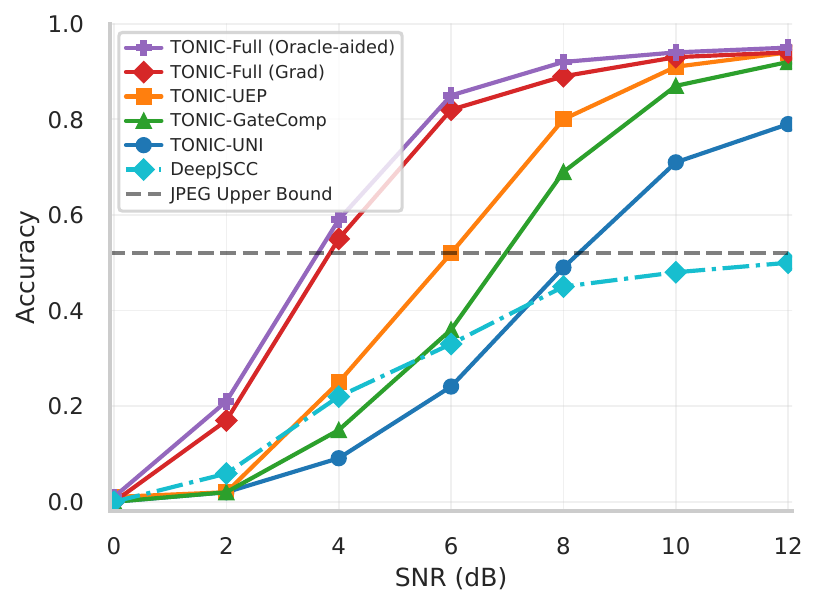}
        \label{fig:main_acc_rician}
    }
    \caption{Accuracy versus SNR under AWGN, Rayleigh fading, and Rician fading.}
    \label{fig:main_acc_three_channels}
\end{figure*}

\begin{table}[t]
\centering
\caption{Main implementation and evaluation settings.}
\label{tab:exp_config}
\footnotesize
\renewcommand{\arraystretch}{1.05}
\setlength{\tabcolsep}{4pt}
\begin{tabular}{p{0.30\columnwidth} p{0.62\columnwidth}}
\toprule
\textbf{Item} & \textbf{Setting} \\
\midrule
Task & Image classification \\
Datasets & CIFAR-10 (sanity check), ImageNet-100 (main) \\
Tokenizer & Pretrained VQ-style visual tokenizer \\
Codebook size & $K=16384$ \\
Token grid / length & $24\times24$, $L=576$ \\
Bits per token & $m=\lceil \log_2 K \rceil = 14$ \\
Task head & Lightweight token-based classifier; offline-trained and frozen \\
Completion model & Transformer-based masked-token predictor; offline-trained and frozen \\
Utility model & Gradient-based utility (deployable), masking-based utility (oracle-aided reference) \\
PHY & 16QAM + group-wise LDPC rate selection \\
Nominal budget & $B_0=4096$ symbols \\
Budget sweep & $\{0.5B_0,\, B_0,\, 2B_0\}$ \\
Channels & AWGN, Rayleigh block fading, Rician block fading \\
Metrics & Accuracy, TER, WAR \\
Receiver assumption & Coherent demodulation/decoding with CSI \\
\bottomrule
\end{tabular}
\end{table}

\section{Performance Evaluation}
\label{sec:perf}
This section evaluates TONIC under a fixed symbol-budget constraint and a practical digital PHY. We report both task-level performance and token-level diagnostics, and benchmark TONIC against representative pixel-domain baselines, separation-style references, and token-domain variants under matched communication resources. The evaluation addresses four questions: whether TONIC improves task accuracy across different wireless channels, whether the gain persists across different communication budgets, how utility-aware protection and confidence-aware gating affect token-level behavior, and whether the qualitative evidence is consistent with the intended design logic of the framework.

\subsection{Experimental Setup}
\label{subsec:setup}

\subsubsection{Datasets and task}
\label{subsubsec:data_task}

We evaluate TONIC on image classification, which serves as the downstream task throughout this section. While the proposed token-centric communication framework is not conceptually restricted to classification, the present experiments focus exclusively on this setting. CIFAR-10 is used for sanity checks and debugging, while ImageNet-100 serves as the main evaluation dataset. The principal task metric is classification accuracy, and we also record the average task loss.

\subsubsection{Tokenizer and token interface}
\label{subsubsec:tokenizer}
Each image is converted into a discrete token sequence by a pretrained VQ-style tokenizer. Let $\mathbf{t} = [t_1,\ldots,t_L]^{\mathsf T}$ denote the source token sequence, where $t_i \in \mathcal{K} = \{1,\ldots,K\}$. For digital transmission, each token index is mapped to a fixed-length bit representation with $m=\lceil \log_2 K \rceil$ bits. Unless otherwise stated, the downstream task head operates on the completed token embeddings rather than on reconstructed pixels.

\subsubsection{Completion model and offline training}
\label{subsubsec:completion_setup}
The receiver employs a Transformer-based masked-token completion model $f_{\mathrm{comp}}(\cdot)$. This model is trained offline on tokenized images using random masking and cross-entropy loss, and is frozen during all communication experiments. The downstream task head is a lightweight token-based classifier operating on the completed token embeddings. Both the completion model and the task head are trained offline and remain fixed throughout evaluation; no per-SNR or per-channel end-to-end retraining is performed.

\subsubsection{Utility profile, grouping, and threshold calibration}
\label{subsubsec:utility_grouping}

To avoid per-sample signaling overhead, TONIC uses a shared position-wise utility profile estimated offline from the calibration data. Token positions are quantized into $G$ utility groups, and the group-wise confidence thresholds $\{\tau_g\}$ are calibrated offline by minimizing validation task loss, as described in Algorithm~\ref{alg:tonic_tau_calibration}. This yields a compact control interface shared by the transmitter and receiver.

\subsubsection{PHY instantiation}
\label{subsubsec:phy_setup}

The physical layer is instantiated using fixed 16QAM modulation and group-wise unequal protection realized through LDPC code-rate selection from a finite candidate set $\mathcal{P}$. Each sample is transmitted under a fixed symbol budget $N$ and an average transmit-power constraint. Soft decoding produces bit-level log-likelihood ratios, which are then mapped to token posteriors and confidence scores used by the receiver-side gating rule. The main implementation and evaluation settings are summarized in Table~\ref{tab:exp_config}.

\subsubsection{Channel models}
\label{subsubsec:channels}

We evaluate TONIC under three instantiations of the flat block-fading channel model introduced in Section~III: AWGN, Rayleigh block fading, and Rician block fading with a fixed $K$-factor. In all cases, the channels are normalized such that $\mathbb{E}[|h|^2]=1$, so that the nominal SNR remains comparable across channel types. The receiver is assumed to have CSI for coherent demodulation and decoding. These channel instantiations are used to test whether the proposed token-centric design remains effective beyond a single propagation condition.

\subsection{Baselines and TONIC Variants}
\label{subsec:baselines}

We compare TONIC against both external references and internal ablations.

\subsubsection{External baselines}

\begin{itemize}
    \item JPEG Upper Reference: This is a budget-constrained ideal-link reference obtained by selecting the best JPEG operating point under the same nominal communication budget without channel corruption. It is reported as a separation-style upper reference under the considered budget rather than as a noisy-channel baseline.
    \item DeepJSCC: We include a pixel-domain deep joint source channel coding baseline evaluated under the same nominal communication budget. This baseline provides a representative learned image-transmission benchmark outside the token domain.
\end{itemize}

\subsubsection{TONIC variants}

The following TONIC variants are reported to isolate the contribution of each module.
\begin{itemize}
    \item TONIC-UNI: uniform protection only, with no confidence gating and no completion.
    \item TONIC-UEP: utility-aware unequal protection only, without receiver-side gating or completion.
    \item TONIC-GateComp: uniform protection combined with confidence-aware gating and completion.
    \item TONIC-Full (Grad): the full practical design, combining utility-aware protection, confidence-aware gating, and completion using the gradient-based utility profile.
    \item TONIC-Full (Oracle-aided): the same full pipeline, but with an oracle-aided utility profile used as an offline reference.
\end{itemize}

These variants are designed to disentangle the gain of transmitter-side unequal protection from that of receiver-side gating and completion-assisted recovery.

\subsection{Metrics}
\label{subsec:metrics}

We report a compact set of metrics that matches the token-centric perspective of TONIC while keeping the evaluation focused.

\subsubsection{Task metrics}
\label{subsubsec:task_metrics}

The main task metric is classification accuracy. We also record the average task loss.

\subsubsection{Token error rate (TER)}
\label{subsubsec:ter_metric}

Let $\mathbf{t} = [t_1,\ldots,t_L]^{\mathsf T}$ denote the source token sequence and let $\bar{\mathbf{t}} = [\bar{t}_1,\ldots,\bar{t}_L]^{\mathsf T}$ denote the final completed token sequence after receiver-side gating and completion. We define
\begin{equation}
\mathrm{TER} \triangleq \frac{1}{L}\sum_{i=1}^{L}\mathbbm{1}\{\bar{t}_i \neq t_i\}.
\label{eq:ter_def_v4}
\end{equation}
TER therefore measures the final end-to-end token mismatch after the full TONIC recovery pipeline.

\subsubsection{Wrong-but-accepted ratio (WAR)}
\label{subsubsec:war_metric}

Let $\hat{\mathbf{t}} = [\hat{t}_1,\ldots,\hat{t}_L]^{\mathsf T}$ denote the hard-decoded token sequence before gating, and let $\tilde{\mathbf{t}} = [\tilde{t}_1,\ldots,\tilde{t}_L]^{\mathsf T}$ denote the gated sequence with $\tilde{t}_i \in \mathcal{K}_{\perp}$. We define
\begin{equation}
\mathrm{WAR} \triangleq
\frac{
\sum_{i=1}^{L}\mathbbm{1}\{\tilde{t}_i \neq \perp,\ \hat{t}_i \neq t_i\}
}{
\sum_{i=1}^{L}\mathbbm{1}\{\tilde{t}_i \neq \perp\}
}.
\label{eq:war_def_v4}
\end{equation}
If no token is accepted at a given operating point, WAR is defined to be zero by convention. WAR measures the fraction of erroneous hard-decoded tokens among those accepted by the receiver before completion. It is therefore a receiver-side diagnostic metric rather than a direct surrogate for final task performance.

Unless otherwise stated, the token-domain diagnostics TER and WAR are reported only for the TONIC family. For pixel-domain baselines such as DeepJSCC and for the JPEG upper reference, the main comparison metric is task accuracy, since these methods do not naturally admit the same token-level decomposition used by TONIC.

\subsection{Results and Discussion}
\label{subsec:results}

\subsubsection{Accuracy versus SNR across channels}
\label{subsubsec:acc_snr_channels}

Fig.~\ref{fig:main_acc_three_channels} reports the main performance comparison under AWGN, Rayleigh, and Rician channels at a fixed communication budget. Three observations are particularly important.

First, the TONIC family consistently outperforms the reduced token-domain variants across all three channels. In particular, TONIC-Full (Grad) delivers the strongest practical performance over most of the evaluated SNR range, which confirms that the combination of utility-aware protection and receiver-side confidence-aware gating is more effective than either mechanism alone.
Second, in the evaluated setting, AWGN yields the highest accuracy, Rayleigh fading is the most challenging, and Rician fading lies in between. This behavior is consistent with the different levels of channel uncertainty in the three channel instantiations and shows that the gain of TONIC is not tied to a single propagation condition.
Third, the gap between TONIC-Full (Grad) and TONIC-Full (Oracle-aided) is consistently small. This is a useful result rather than a limitation, because it indicates that the practical gradient-based utility profile already captures most of the token-importance structure needed for transmitter-side protection. In contrast, the gain of TONIC-UEP over TONIC-UNI confirms that unequal protection alone is already beneficial, even before introducing receiver-side gating and completion.
\begin{figure}[t!]
    \centering
    \includegraphics[width=0.8\columnwidth]{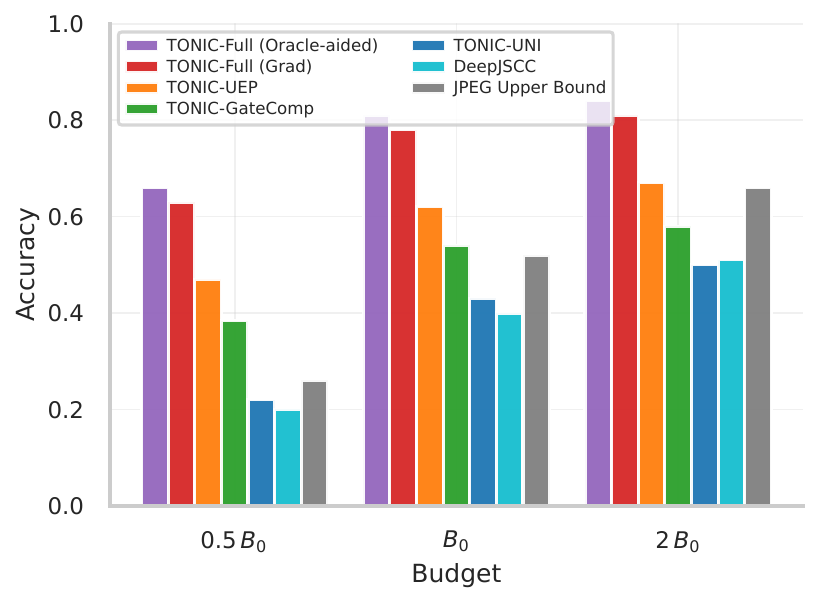}
    \caption{Accuracy versus communication budget under Rayleigh fading.}
    \label{fig:acc_vs_budget_rayleigh}
\end{figure}
\begin{figure}[t!]
    \centering
    \subfigure[TER versus SNR.]{
        \includegraphics[width=0.8\columnwidth]{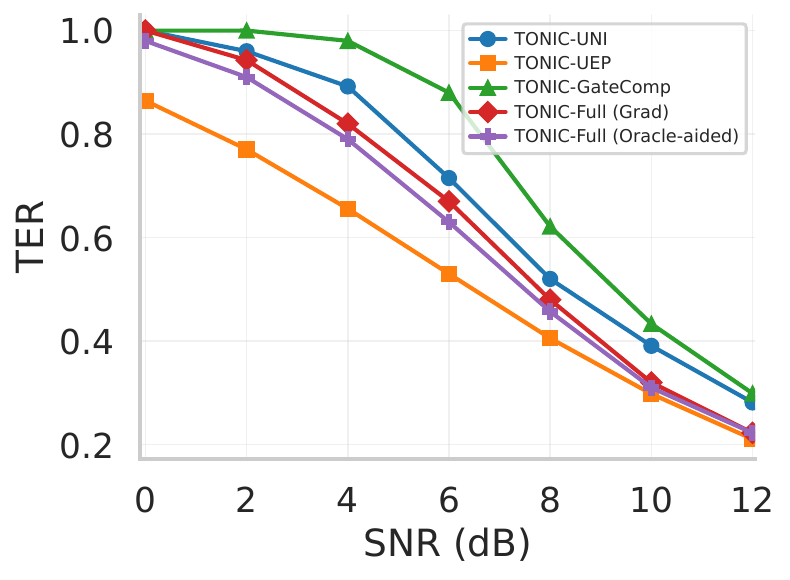}
        \label{fig:ter_vs_snr_rayleigh}
    }
    \subfigure[WAR versus SNR.]{
        \includegraphics[width=0.8\columnwidth]{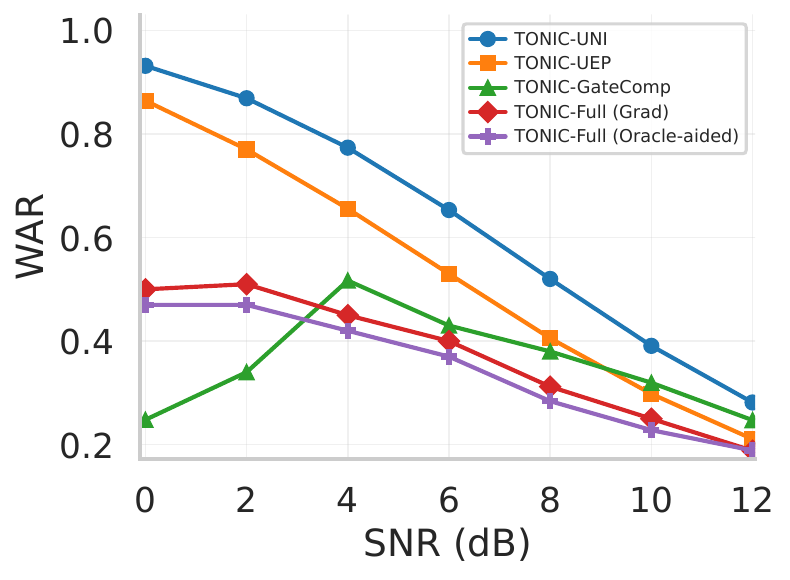}
        \label{fig:war_vs_snr_rayleigh}
    }
    \caption{Token-level reliability of TONIC under Rayleigh fading.}
    \label{fig:ter_war_rayleigh}
\end{figure}

\begin{figure}[t!]
    \centering
    \subfigure[Accuracy versus TER.]{
        \includegraphics[width=0.8\columnwidth]{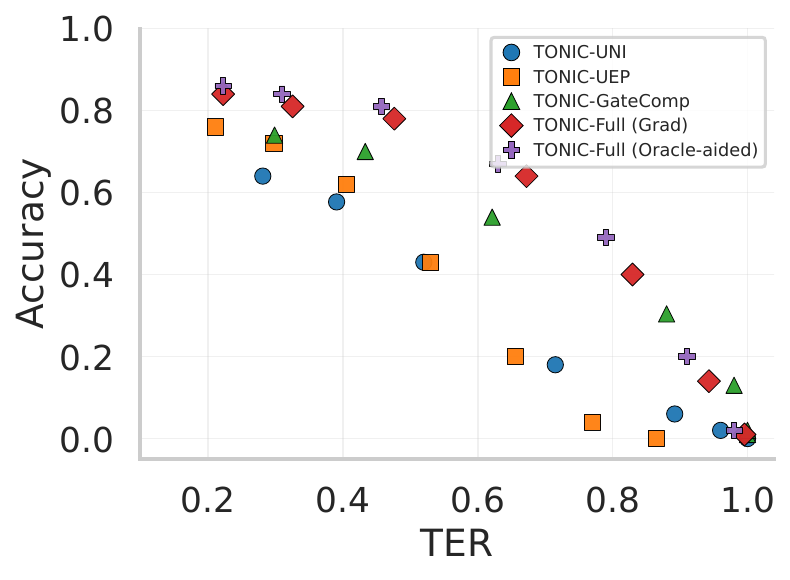}
        \label{fig:acc_vs_ter_rayleigh}
    }
    \subfigure[Accuracy versus WAR.]{
        \includegraphics[width=0.8\columnwidth]{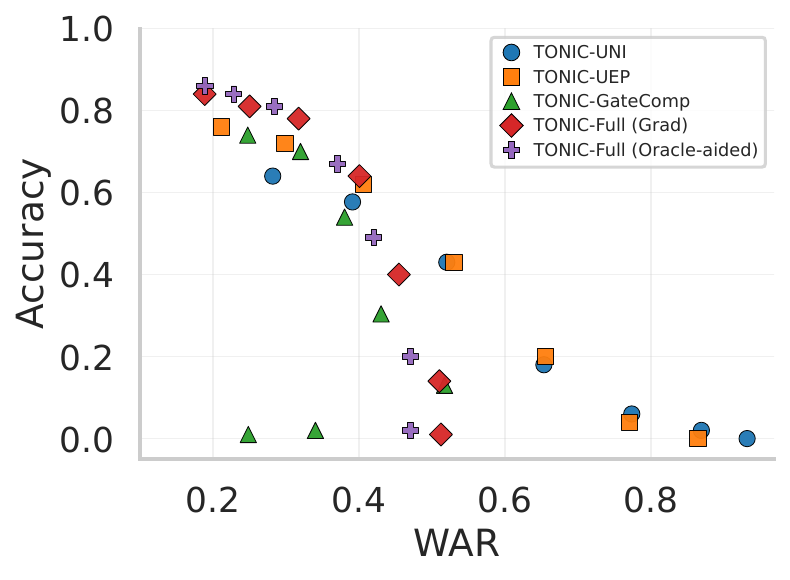}
        \label{fig:acc_vs_war_rayleigh}
    }
    \caption{Task accuracy versus token-level error metrics under Rayleigh fading}
    \label{fig:acc_vs_ter_war_rayleigh}
\end{figure}

\begin{figure*}[t!]
    \centering
    \subfigure[Original image.]{
        \includegraphics[width=0.166\textwidth]{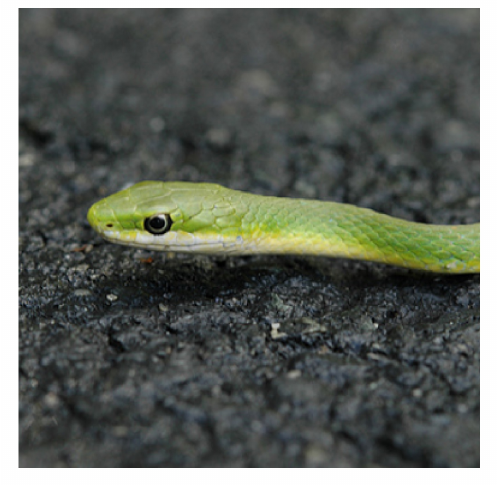}
        \label{fig:case_original}
    }
    \hfill
    \subfigure[Token utility heatmap.]{
        \includegraphics[width=0.2\textwidth]{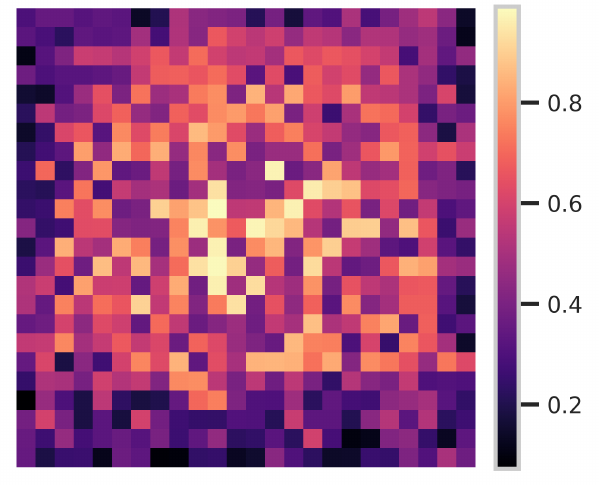}
        \label{fig:case_utility}
    }
    \hfill
    \subfigure[Utility-aware grouping map.]{
        \includegraphics[width=0.2\textwidth]{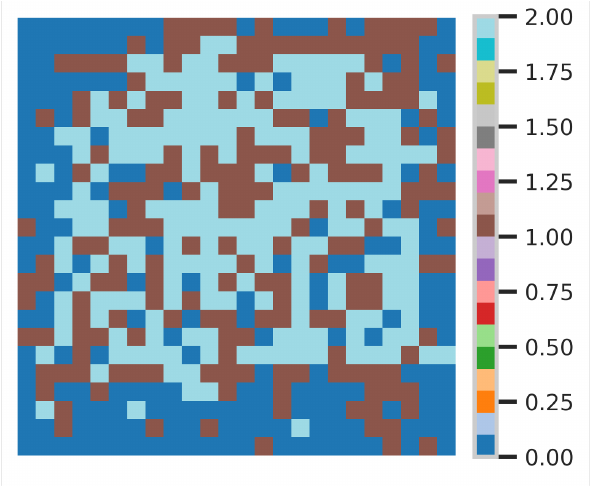}
        \label{fig:case_grouping}
    }
    \hfill
    \subfigure[Per-group protection and behavior.]{
        \includegraphics[width=0.25\textwidth]{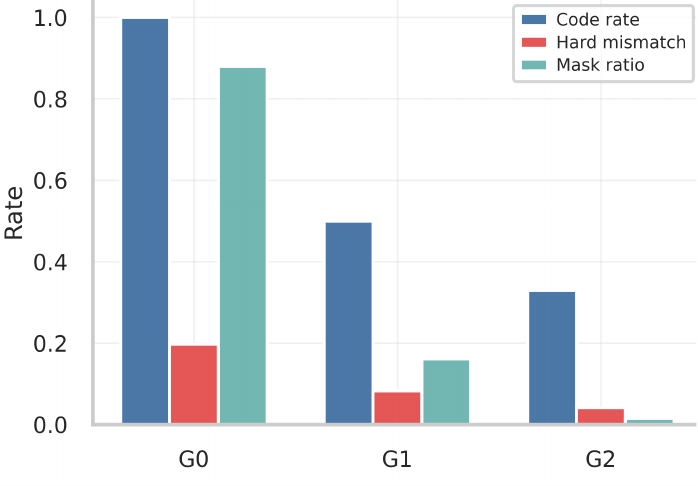}
        \label{fig:case_group_behavior}
    }
    \caption{Illustration of the utility-aware grouping and protection mechanism of TONIC for a representative sample.}
    \label{fig:utility_grouping_visualization}
\end{figure*}

\begin{figure*}[t!]
    \centering
    \includegraphics[width=0.7\textwidth]{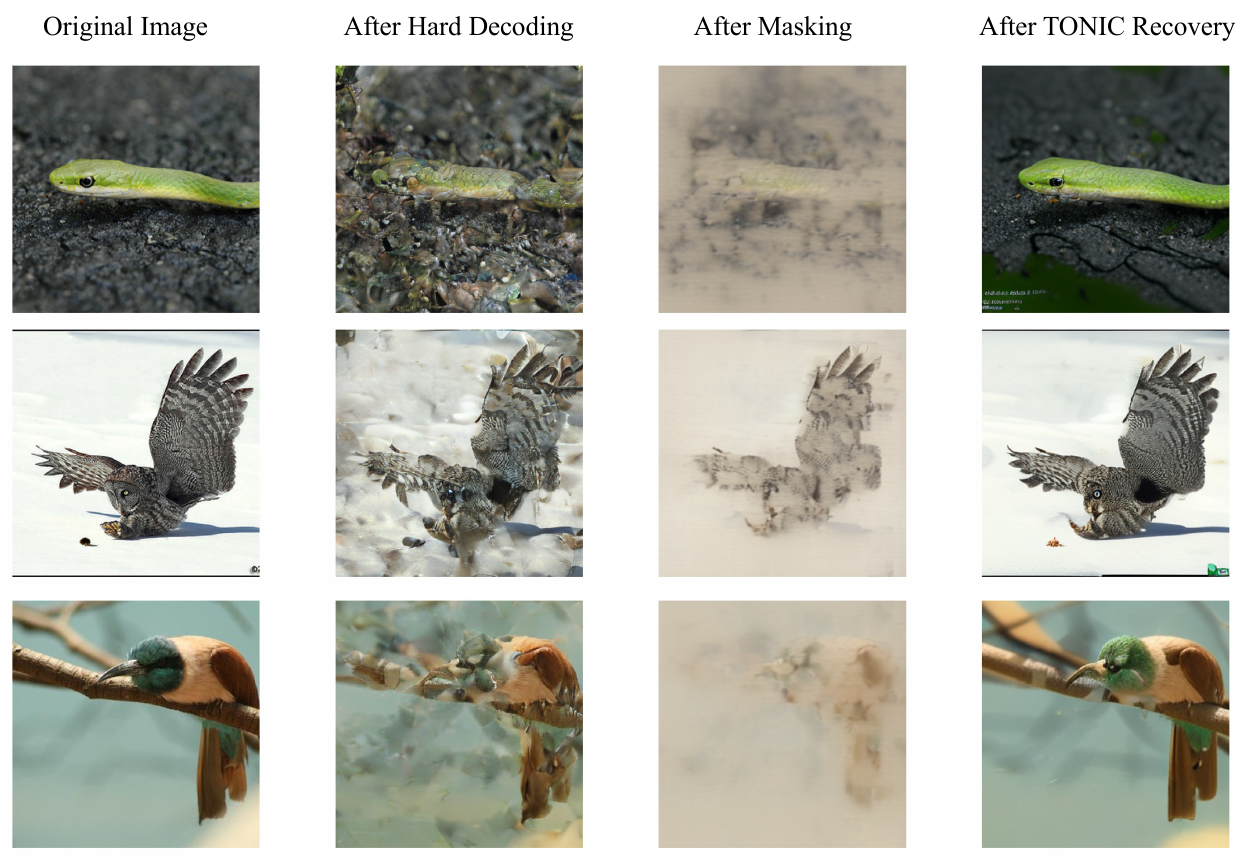}
    \caption{Detokenized images along the TONIC recovery pipeline, shown for qualitative intuition only; visual fidelity is not the optimization target.}
    \label{fig:detok_pipeline}
\end{figure*}

\subsubsection{Accuracy versus communication budget}
\label{subsubsec:budget_sweep}

Fig.~\ref{fig:acc_vs_budget_rayleigh} evaluates how task accuracy scales with the communication budget under Rayleigh fading. The main message is that the gain of TONIC is not restricted to a single carefully chosen operating point. Instead, utility-aware protection and receiver-side completion remain effective across different resource regimes.

The advantage of TONIC is particularly meaningful in the low-budget regime, where uniformly protecting all token positions is inefficient and the value of prioritizing task-critical positions becomes more pronounced. As the budget increases, all methods improve, as expected. However, the relative advantage of TONIC-Full remains visible, suggesting that the gain of TONIC is not tied to a single budget point.

\subsubsection{Token-level reliability under Rayleigh fading}
\label{subsubsec:token_metrics_snr}

To better understand the mechanism behind the task-level gains, Fig.~\ref{fig:ter_war_rayleigh} reports token-level behavior under the Rayleigh channel at budget $4096$. The two metrics play different roles and should be interpreted jointly.

In Fig.~\ref{fig:ter_vs_snr_rayleigh}, TER decreases with SNR for all TONIC variants, confirming that the final completed token sequence becomes more accurate as channel quality improves. The lower TER achieved by stronger TONIC variants shows that transmitter-side utility-aware protection contributes directly to end-to-end token recovery.

The behavior of WAR in Fig.~\ref{fig:war_vs_snr_rayleigh} is qualitatively different, which is precisely why it is informative. At very low SNR, methods with receiver-side gating can exhibit a small WAR because they reject many unreliable hard token decisions instead of accepting them as harmful substitutions; in this regime, a small WAR should therefore not be interpreted in isolation as evidence of superior end-to-end performance. As the SNR increases, the accepted-token set expands and the receiver begins to trust more positions, so WAR may first rise and then decline. Accordingly, WAR is best interpreted as a diagnostic measure of the purity of accepted hard token decisions before completion, rather than as a direct surrogate for final task accuracy. Taken together, TER and WAR show that TONIC improves performance through two coupled mechanisms: transmitter-side utility-aware protection and receiver-side confidence-aware gating.

\subsubsection{Accuracy versus TER/WAR relationship}
\label{subsubsec:acc_ter_scatter}

Fig.~\ref{fig:acc_vs_ter_war_rayleigh} visualizes task accuracy against token-level error metrics for the TONIC family under Rayleigh fading at budget $4096$. In Fig.~\ref{fig:acc_vs_ter_rayleigh}, lower TER generally corresponds to higher task accuracy, but the relationship is not one-to-one. This is expected because classification performance depends not only on how many token positions are incorrect, but also on which positions are incorrect. Small differences concentrated on task-critical positions can therefore lead to visibly different classification outcomes even when the average TER remains similar.

Fig.~\ref{fig:acc_vs_war_rayleigh} provides a complementary receiver-side perspective. A lower WAR does not automatically imply the highest final accuracy, since WAR only reflects the accepted hard-token set before completion. A method may achieve a low WAR by aggressively erasing uncertain positions while still relying heavily on the completion stage for the final recovery. Hence, WAR and accuracy should be interpreted jointly: WAR captures the quality of accepted hard decisions, whereas final accuracy depends on the complete sequence of acceptance, erasure, completion, and downstream inference.

\subsubsection{Utility-aware grouping and qualitative intuition}
\label{subsubsec:qual_example}

Fig.~\ref{fig:utility_grouping_visualization} provides a qualitative view of the transmitter-side design. The utility heatmap shows that token importance is strongly non-uniform across spatial positions, while the grouping map converts this heterogeneity into a finite number of utility groups. The final panel then shows that different groups receive different protection strengths and exhibit different mismatch and masking behavior. Taken together, these visualizations illustrate the transmitter-side intuition behind TONIC: under a fixed communication budget, protection should be concentrated on positions that are more important to the downstream task.

\subsubsection{Detokenized intuition along the recovery pipeline}
\label{subsubsec:detok_pipeline}

Fig.~\ref{fig:detok_pipeline} provides an additional intuitive view using detokenized images along the transmission-and-recovery pipeline. These images are shown purely for qualitative intuition. Since TONIC is task-oriented rather than reconstruction-oriented, visual fidelity is not the optimization target. Accordingly, the detokenized images should not be interpreted as a reconstruction benchmark; they only illustrate how hard token corruption, erasure gating, and completion affect the token sequence before task inference.

\subsection{Discussion}
\label{subsec:discussion}

The experimental results consistently support the central design logic of TONIC. Utility-aware protection reduces the frequency of harmful corruption on task-relevant token positions, while confidence-aware gating converts part of the remaining uncertainty into a form that is more compatible with completion-assisted recovery. Their combination yields the strongest practical operating point among the TONIC variants.

The results also clarify the role of the oracle-aided reference. The small gap between TONIC-Full (Grad) and TONIC-Full (Oracle-aided) suggests that the practical gradient-based utility profile is already sufficiently informative for protection design. This is favorable from a deployment perspective, since it indicates that most of the gain can be captured without relying on oracle supervision.

Finally, the qualitative figures reinforce an important conceptual point: TONIC should not be judged by reconstructed image fidelity. The communication target is the tokenized semantic representation required by the downstream model, and the value of the framework lies in preserving task-relevant token structure under limited communication resources.

\section{Conclusion}
\label{sec:conclusion}

This paper presented TONIC, a token-centric semantic communication framework for task-oriented wireless systems. TONIC departs from bit-centric communication by directly targeting the semantic token interface consumed by the downstream model. The framework combines transmitter-side utility-aware unequal protection with receiver-side confidence-aware gating and generative completion, thereby jointly controlling which token positions receive stronger protection and how residual uncertainty is handled before task inference.
We further established a utility-aware Bayes-risk interpretation for the receiver-side gating rule and developed a practical deployment pipeline based on offline utility profiling, token grouping, and threshold calibration. Experimental results on image classification showed that TONIC consistently improves task accuracy over separation-based transmission, pixel-domain deep JSCC, and token-domain baselines under matched communication budgets across AWGN, Rayleigh, and Rician channels.
The present work instantiated TONIC on wireless image transmission with downstream classification. Future work will extend the framework to richer multimodal and multiuser settings and further tighten the interaction between communication, token completion, and downstream decision making under dynamic wireless conditions.

\bibliography{Ref_cleaned_ieee}
\bibliographystyle{IEEEtran}

\end{document}